\documentclass[letterpaper, 10 pt, journal, twoside]{IEEEtran} % % RA-L (final version)

\IEEEoverridecommandlockouts

\usepackage[english]{babel}
\usepackage[T1]{fontenc}

\usepackage{microtype}
\usepackage{blindtext}
\usepackage{siunitx}
\usepackage{titlecaps}
\Addlcwords{or with if of and a to -off off}

\usepackage{graphicx}
\usepackage{color}
\graphicspath{{figures/}}
\usepackage[font=small]{caption}
\usepackage{subfloat}
\usepackage{subcaption}
\usepackage{pgfplots}
\usepackage{tikz}
\usepackage{tikzscale}
\usepackage{scalerel}
\usetikzlibrary{arrows}
\usetikzlibrary{plotmarks}
\usetikzlibrary{svg.path}
\usetikzlibrary{shapes.multipart}
\pgfplotsset{compat=newest}
\pgfplotsset{every axis/.append style={
		label style={font=\Large},
		tick label style={font=\large}  
}}
\tikzstyle{int}=[draw, fill=black!10, minimum size=5em,thick]
\tikzstyle{init} = [pin edge={to-,thick,black}]
\tikzset{Node Style/.style={thick, draw,fill=white,circle,inner sep=0pt,minimum size=3pt}}
\usepackage{amsmath,amsfonts,amssymb,amsthm}
\usepackage[short,c2]{optidef}
\usepackage{float}
\usepackage{xspace}
\usepackage{siunitx}

\usepackage{array}
\usepackage{stfloats}

\usepackage{hhline}
\usepackage[export]{adjustbox}
\usepackage{csvsimple}

\usepackage[colorinlistoftodos]{todonotes}
\usepackage{comment}
\newcommand{\subparagraph}{} % temporary fix for titlesec error
\usepackage{titlesec}

\usepackage[linesnumbered,ruled,vlined]{algorithm2e}
\usepackage{multirow}
\usepackage{multicol}
\makeatletter
\gdef\Shortstack{\@ifnextchar[\@Shortstack{\@Shortstack[c]}}
\gdef\@Shortstack[#1]#2{%
	\leavevmode
	\vbox\bgroup
	\baselineskip-\p@\lineskip 3\p@
	\let\mb@l\hss\let\mb@r\hss
	\expandafter\let\csname mb@#1\endcsname\relax
	\let\\\@stackcr\setlength{\baselineskip}{#2}%
	\@ishortstack}
\makeatother
\usepackage{booktabs}
\usepackage{tabularx}
\usepackage[absolute]{textpos}

\usepackage{cite}
\newcommand{\citep}[1]{\cite{#1}}

\usepackage{url}
\makeatletter
\let\NAT@parse\undefined
\makeatother
\usepackage[bookmarks=true,colorlinks]{hyperref}
\hypersetup{citecolor=blue}
\usepackage[capitalize]{cleveref}
\usepackage{orcidlink}

\titlespacing*{\section}{0pt}{4mm}{2mm}
\titlespacing*{\subsection}{0pt}{2mm}{2mm}

\newcommand{\myParagraph}[1]{{\bf #1.}\xspace}

\newcommand{\isExtended}[2]{#2} % #2 : arxiv version

\newcommand{\bw}{{B}}
\newcommand{\eg}{\emph{e.g.,}\xspace}
\newcommand{\ie}{\emph{i.e.,}\xspace}
\newcommand{\vs}{\emph{vs.}\xspace}
\newcommand{\name}{D-Lite\xspace}
\newcommand{\namebu}{BUD-Lite\xspace}
\newcommand{\nametd}{TOD-Lite\xspace}

\newcommand{\dsgname}{3D Scene Graph\xspace}
\newcommand{\dsgnames}{3D Scene Graphs\xspace}
\newcommand{\dsgshortname}{DSG\xspace}
\newcommand{\dsgshortnames}{DSGs\xspace}

\DeclarePairedDelimiter\ceil{\lceil}{\rceil}

\theoremstyle{plain}

\newtheorem{prop}{Proposition}

\theoremstyle{definition}
\newtheorem{definition}{Definition}

\newtheorem{ass}{Assumption}
\theoremstyle{remark}

\newcommand{\bmat}{\left[ \begin{array}}
\newcommand{\emat}{\end{array} \right]}
\newcommand{\bal}{\begin{align}}
\newcommand{\eal}{\end{align}}

\newcommand{\lr}{\left(}
\newcommand{\rr}{\right)}
\newcommand{\ls}{\left[}
\newcommand{\rs}{\right]}
\newcommand{\lb}{\left\lbrace}
\newcommand{\rb}{\right\rbrace}

\newcommand{\eps}{\varepsilon}
\newcommand{\tmax}{T_{\text{max}}}
\newcommand{\dsg}{\mathcal{G}}
\newcommand{\nodes}[1][]{\mathcal{V}_{#1}}
\newcommand{\edges}[1][]{\mathcal{E}_{#1}}
\newcommand{\dsgc}{\mathcal{G}'}
\newcommand{\nodesc}{\nodes[\dsgc]}
\newcommand{\edgesc}{\edges[\dsgc]}
\newcommand{\source}{s}
\newcommand{\target}{t}

\newcommand{\sources}{\mathcal{S}}
\newcommand{\targets}{\mathcal{T}}
\newcommand{\pairs}{\mathcal{P}}
\newcommand{\dist}[3]{d_{#1}(#2,#3)}
\newcommand{\distortion}{\beta}
\newcommand{\Out}[1]{\operatorname{Out}(#1)}
\newcommand{\In}[1]{\operatorname{In}(#1)}
\newcommand{\weight}[3]{W^{#1}(#2,#3)}
\newcommand{\maxweight}[3]{W_\text{max}^{#1}(#2,#3)}
\renewcommand{\path}[3]{P_{#1}(#2,#3)}
\makeatletter
\newcommand{\dict}[1][\@empty]{\mathcal{D}%
	\ifx\@empty#1 \else {\ls#1\rs} \fi}
\makeatother

\newcommand{\hierarchy}{\mathcal{H}}
\newcommand{\layer}[2]{\mathcal{L}_{#2}^{#1}}

\newcommand{\ancestor}[3]{a_{#3}^{#2}(#1)}
\newcommand{\parent}[2]{p_{#2}(#1)}
\newcommand{\children}[1][]{C_{#1}}
\newcommand{\diam}[2]{\mathrm{diam}_{#2}(#1)}
\newcommand{\wmin}[1]{W_{\text{min}}^{#1}}
\newcommand{\wmax}[1]{W_{\text{max}}^{#1}}
\newcommand{\wcmin}[1]{W_{\text{min}}^{#1-1,#1}}
\newcommand{\wcmax}[1]{W_{\text{max}}^{#1-1,#1}}
\newcommand{\mmin}[1]{u_{\text{min}}^{#1}}
\newcommand{\mmax}[1]{u_{\text{max}}^{#1}}

\newcommand{\bmin}[1]{\mathrm{diam}_{\text{min}}^{#1}}
\newcommand{\lz}{\ell_0}

\newcommand{\lmax}{\ell_{\text{max}}}

\SetCommentSty{mycommfont}

\newcommand{\LC}[1]{{\color{red}LC: #1}} % Luca C
\newcommand{\blue}[1]{{\color{blue}#1}}

\newcommand{\revise}[1]{#1\xspace}

 \newcommand{\reviseral}[1]{#1\xspace}

\newcommand{\linkToPdf}[1]{\href{#1}{\blue{(pdf)}}}
\newcommand{\linkToPpt}[1]{\href{#1}{\blue{(ppt)}}}
\newcommand{\linkToCode}[1]{\href{#1}{\blue{(code)}}}
\newcommand{\linkToWeb}[1]{\href{#1}{\blue{(web)}}}
\newcommand{\linkToVideo}[1]{\href{#1}{\blue{(video)}}}
\newcommand{\linkToMedia}[1]{\href{#1}{\blue{(media)}}}
\newcommand{\award}[1]{\xspace} % omit awards
\addto\extrasenglish{}
\addto\extrasenglish{}
\addto\extrasenglish{}

\newcommand{\setal}{~\emph{et al.}}
\newcommand{\robotQuery}{r_1}
\newcommand{\robotRespond}{r_2}

\title{
	\name: Navigation-Oriented Compression of \dsgnames for Multi-Robot Collaboration}

\author{
	Yun~Chang\textsuperscript{\orcidlink{0000-0002-2829-5256}},~\IEEEmembership{Graduate~Student~Member,~IEEE},
	Luca~Ballotta\textsuperscript{\orcidlink{0000-0002-6521-7142}},~\IEEEmembership{Member,~IEEE},
	and~Luca~Carlone\textsuperscript{\orcidlink{0000-0003-1884-5397}},~\IEEEmembership{Senior~Member,~IEEE}
	\thanks{
		Manuscript received: April 21, 2023; Revised: August 2, 2023; Accepted: September 8, 2023.
	}
	\thanks{
		This letter was recommended for publication by Editor Ani M. Hsieh upon evaluation of the Associate Editor and Reviewers' comments.
		This work was supported by
		ARL DCIST CRA W911NF-17-2-0181, 
		ONR RAIDER N00014-18-1-2828, 
		Lincoln Laboratory’s Resilient Perception in Degraded Environments program,
		the CARIPARO Foundation Visiting Programme ``HiPeR'', 
		and the Italian Ministry of Education, University and Research (MIUR) through the PRIN Project under Grant 2017NS9FEY ``Realtime Control of 5G Wireless Networks''.
		\textit{(Yun Chang and Luca Ballotta contributed equally to this work.)}
	}
	\thanks{
		Yun Chang and Luca Carlone are with Laboratory for Information \& Decision Systems (LIDS),
		Massachusetts Institute of Technology, Cambridge, MA, USA
		(e-mail: yunchang@mit.edu; lcarlone@mit.edu).
	}
	\thanks{
		Luca Ballotta is with Department of Information Engineering,
		University of Padova, Padova, Italy
		(e-mail: ballotta@dei.unipd.it).
	}
	\thanks{Digital Object Identifier (DOI): see top of this page.}
}

\isExtended{
	\markboth{IEEE Robotics and Automation Letters. Preprint Version. Accepted September, 2023}
	{Chang \MakeLowercase{\textit{et al.}}: \name: Navigation-Oriented Compression of \dsgnames for Multi-Robot Collaboration}}{
}

\begin{document}
	
	\maketitle
	
	\isExtended{}{
		\begin{textblock}{10}(3,.05)
			\begin{center}
				This article has been accepted for publication at the IEEE Robotics and Automation Letters.\\
				Please cite as: Yun Chang, Luca Ballotta, and Luca Carlone,\\
				``D-Lite: Navigation-Oriented Compression of 3D Scene Graphs for Multi-Robot Collaboration'',\\
				in \textit{IEEE Robotics and Automation Letters}, 2023.
			\end{center}
		\end{textblock}
	}
	
	%!TEX root = ../main.tex

\begin{abstract}
	For a multi-robot team that collaboratively explores an unknown environment,
	it is of vital importance that the collected information is efficiently shared among robots
	in order to support
	%boost performance of 
	exploration and navigation tasks.
	Practical constraints of wireless channels,
	such as limited bandwidth, % and bit-rate,
	urge robots to carefully select information to be transmitted.
	In this letter, we consider the case where environmental information is modeled using a \emph{\dsgname},
	a hierarchical map representation that describes both geometric and semantic aspects of the environment.
	Then, we leverage graph-theoretic tools, namely \textit{graph spanners},
	to design \revise{greedy algorithms} that efficiently compress \dsgnames\xspace
	\reviseral{with the aim of enabling} communication \reviseral{between robots} under bandwidth constraints.
	Our compression \revise{algorithms} are \emph{navigation-oriented} in that they are designed to 
	approximately preserve shortest paths between locations of interest
	while 
	meeting a user-specified communication budget constraint.  
	%which are both effective for navigation and
	%appealing under poor communication resources.
	The effectiveness of the proposed algorithms is demonstrated
	% via extensive numerical analysis and on 
	in robot navigation experiments in a realistic simulator.
	\isExtended{}{A video abstract is available at \url{https://youtu.be/nKYXU5VC6A8}.}
\end{abstract}

\isExtended{
}{
}
	%!TEX root = ../main.tex

\section{Introduction}
\label{sec:introduction}

%% Motivation and applications
\IEEEPARstart{I}{n the} near future, robot teams will perform cooperative tasks in a multitude of application scenarios,
ranging from exploration of subterranean environments,
to search-and-rescue missions in hazardous settings, % such as mountains of unsafe buildings,
to human assistance in houses, airports, factory floors, and malls, to mention a few.

A key requirement for cooperative exploration and navigation in an initially unknown environment is to 
build a map model of the environment as the robots explore it.
Recent work has proposed \emph{\dsgnames} as an expressive hierarchical model of complex environments~\cite{Rosinol20rss-dynamicSceneGraphs,Hughes22rss-hydra,Armeni19iccv-3DsceneGraphs,Wu21cvpr-SceneGraphFusion,Kim19tc-3DsceneGraphs,Talak21neurips-neuralTree}:
a \dsgname organizes
% Research efforts over recent years have tackled this challenge from several standpoints.
% In particular,
% a key tool that has been emerging as a leader paradigm for autonomous navigation is the \dsgname:
% this is a hierarchical representation of the environment,
% whereby 
spatial and semantic information, including objects, structures (\eg walls), places (\ie free-space locations the robot can reach), rooms, and buildings into a graph with multiple layers corresponding to different levels of abstraction. %layered graph with different levels of abstractions.
% and free-space 
%encoded by objects and structures (such as walls)
% is used to organize a map at different levels of abstractions,
% from objects,
% to free-space segments allowing a robot to move around,
% to cognition about human-friendly concepts such as rooms and floors
\dsgnames provide a user-friendly model of the scene that can support the execution of high-level instructions by a human.
Also, they capture traversability between places, rooms, and buildings that can
be used for path planning.
% for instance allowing a person to ask a robot to bring her a cup of coffee
% rather than providing Euclidean coordinates of the cup.
% machine together with detailed instructions to use them. 
% \dsgnames and their enhanced version embedding dynamic elements,
% 3D Dynamic Scene Graphs,
% can play a key role in enabling high-level planning for robotic tasks,
% essential for massive deployment within the human society:
% for instance,
% it is way easier for a person to ask a robot for a cup of coffee
% rather than providing Euclidean coordinates of cup and coffee machine
% together with detailed instructions to use them.

\begin{figure}
	\centering
	\includegraphics[trim={30 30 30 30},clip, width=1.0\columnwidth]{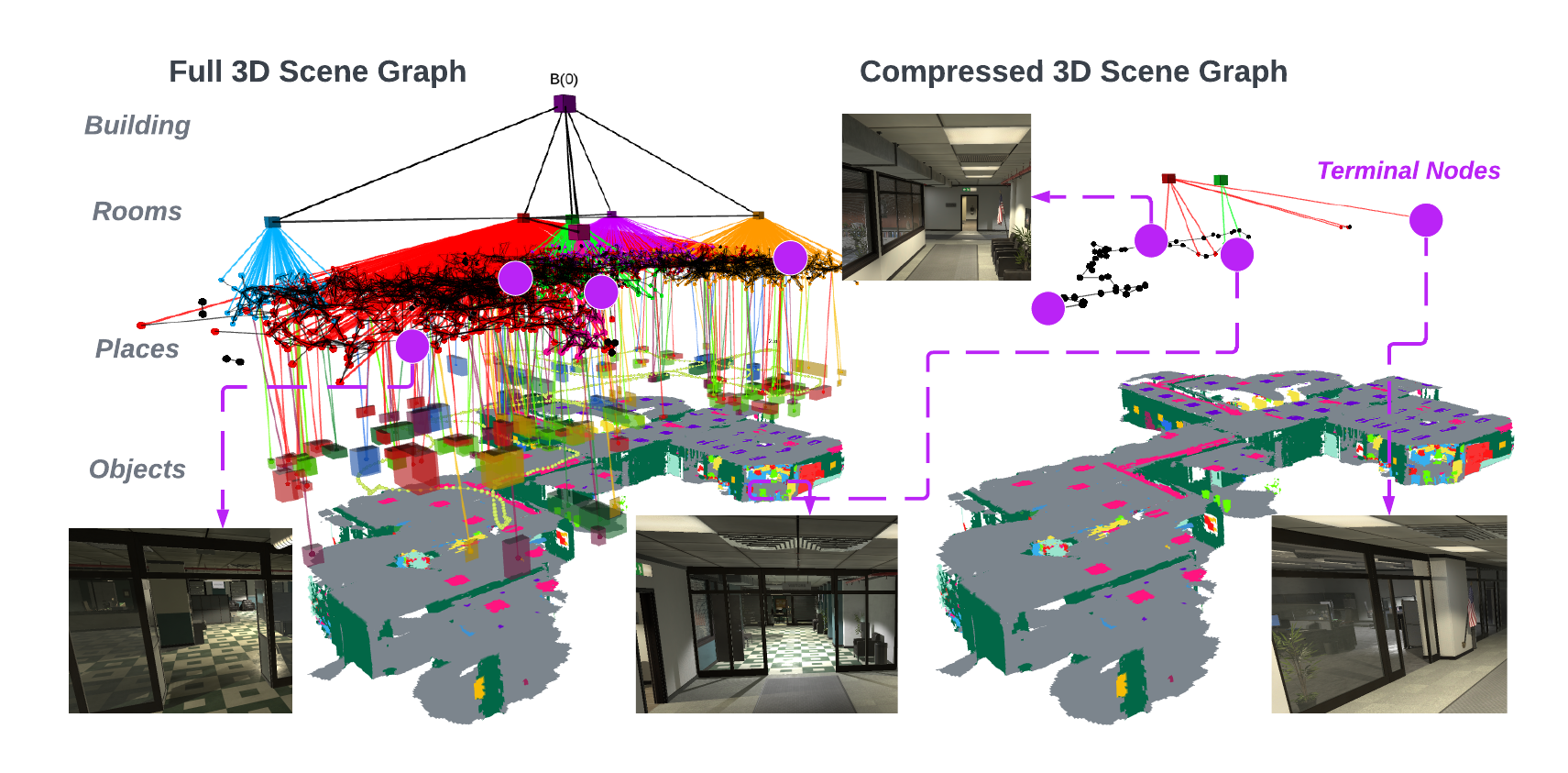} 
	\caption{\dsgname of an environment (left) and compressed version produced by \name (right).
		The purple circles mark the \emph{terminal nodes}: \name approximately preserves \revise{shortest-path distances}
		between those locations of interest.}
	\label{fig:cover}
\end{figure}

To scale up from single- to multi-robot systems \revise{and to longer missions and larger environments},
a key challenge is to share the map information among the robots to support cooperation.
For instance, the robots may exchange partial maps such that a robot can navigate
within a portion of the environment mapped by another robot. 
% coordination in large-scale environments .
%at such high level of abstraction.
% In particular,
% a team of robots exploring an environment should be able to internally share information
% to enable fast and efficient exploration.
However, 
%in practical application such information exchange is subject to 
% a limiting factor encountered in practical applications is the presence of 
% communication constraints: % constrained communication resources:
the potentially high volume of data to be transferred over a shared wireless channel
easily saturates the available bandwidth,
%congestion the latter,
degrading team performance.
%\LC{we will need to be more quantitative here or point to a paper observing this issue}
%\marginnote{LB: One major comment by reviewers was that one usually has enough bandwidth in structured environments,
%	while the opposite holds for unstructured environments which may not be modeled thru a DSG.
%	Maybe it's worth mentioning why this is not the case,
%	for example stressing that such bandwidth limitations may occur even in structured environments and/or in places with good access to the Internet or some communication channel (eg in cities).}
\revise{This holds true especially when the wireless channel is also used to transmit other information in the field
 ---such as images or place recognition information for localization and map reconstruction---
 which further limits the bandwidth available for transmitting map information in a timely manner~\cite{Cieslewski18icra,Tian18rss,Chang21icra-KimeraMulti,Dennison22ral-loopPrioritization}.}
The challenge of information sharing is particularly relevant when the map is modeled as a \dsgname, 
since these are rich and 
potentially large models if all nodes and edges are retained. 
% A notable case is represented precisely by \dsgnames:
% these data typically have large size,
% and sharing them during online operation may be prohibitive under communication constrains.
On the other hand, \dsgnames also provide opportunities for compression: for instance, 
the robots may exchange information about rooms in the environment rather than sharing fine-grained traversability information
encoded by the place layer; similarly, for a large-scale scene, the robot may just specify a sequence of buildings to be traversed, abstracting away geometric information at lower levels.
This is similar to what humans do: 
when providing instructions to a person about how to reach a location in a building,
we would specify a sequence of rooms and landmarks (\eg objects or structures) rather than a detailed metric map or a precise path.

Therefore, the question we address in this paper is: \emph{how can we compress a \dsgname to retain relevant information  
the robots can use for navigation while meeting a communication budget constraint, expressed as the maximum size of the map the robots can transmit?} 
% On the other hand,
% the compressed \dsgname needs to retain relevant information to be used by a robot requesting that map,
% raising a nontrivial trade-off between compression quality and rate.
Besides multi-robot communication, task-driven map compression can play a role in long-term autonomy under resource constraints, 
where the robots might suffer memory limitations and %be unable to store a large-scale map due to and
retain only key portions of a large map.
%\LC{break into 2 sentences}
\revise{Such a compression is also useful when it is desirable for the robots to share essential information under privacy considerations 
by sending only task-relevant data~\cite{Choudhary17ijrr-distributedPGO3D}}. %  while limiting the sharing at the expense of optimality
% such as mobile ground robots or UAVs,
% which may have limited memory capacity.
% retaining only the essentially information is a 
% keeping only essential information for navigation task allows to dramatically decrease storage requirements,
% so that large environments can be progressively explored and mapped on-board robots without running out of memory.
%Hence,
%suitable strategies are needed to enable an effective compression
			%!TEX root = ../main.tex

\myParagraph{Related work}
\isExtended{%
	% \LC{kasra: talk efficiently to me}
	Graph compression is an active area of research in 
	mathematics, computer science, and telecommunications, where
	it finds applications to, \eg
	vehicle and packet routing~\cite{Becker17esa-vehicleRouting,Gilbertkdd-compressingGraphs,Dobson14ijrr-roadmapSpanners},
	and 3D point cloud compression~\cite{DeQueiroz16ip-pointCloudCompression,Sun19ral-pointcloudCompression,Thanou16ip-pointcloudCompression}.
}{%
	Graph compression is active area of research in discrete
	mathematics, computer science, and telecommunications, where
	it finds applications to, \eg
	vehicle routing~\cite{Becker17esa-vehicleRouting,Dobson14ijrr-roadmapSpanners}, packet routing in wireless networks~\cite{Gilbertkdd-compressingGraphs},
	and compression of unstructured data such as 3D point
	clouds~\cite{DeQueiroz16ip-pointCloudCompression,Sun19ral-pointcloudCompression,Thanou16ip-pointcloudCompression}.
}
%with have been especially pushed by application domains involving environments modeled as large graphs,
% from 
% \isExtended{}{, packet routing in wireless networks~\cite{Gilbertkdd-compressingGraphs},}%
%in logistics and transportation
% to more recent applications in computer science and telecommunications such as

A prominent body of works simplifies a graph by carefully pruning it %removing edges
\isExtended{%
	to retain relevant information.}{%
	based on structural properties of the graph.
	These methods typically entail some information loss, and
	%clearly cause information loss,
	aim to only retain relevant information
	when storing,
	processing,
	parsing,
	or transmitting the full graph is infeasible.}
For example,
references~\cite{Raghavan03-WebGraphs,Gilbertkdd-compressingGraphs} find efficient representations of huge web and communication networks
by heuristically selecting a few key elements,
while the work~\cite{Chekuri15esa-connectivityPreserving} prunes graphs while preserving connectivity among nodes.
Within the discrete mathematics literature,
% tailored structures are studied for 
graph compression has been studied with focus on ensuring low distortion (or \emph{stretch}) of 
inter-node distances.
For example,
\emph{spanning trees} and \emph{Steiner trees} are the smallest subgraphs maintaining connectivity
in undirected graphs~\cite{Harish12gpuc-spanningTrees,Mehlhorn88ipl-steinerTreeFast}.
\emph{Graph spanners} remove a subset of edges while allowing for 
a user-defined maximum distortion of shortest paths~\cite{Ahmed20csr-GraphSpannersReview,Kobayashi20stacs-fptAlgSpanner,Elkin22dc-WeightedAdditiveSpanner}.
A special case are \emph{distance preservers}~\cite{Bodwin19soda-distancePreservers}
that prune graphs but keep unaltered the distances for specified node pairs.
\emph{Emulators} are tools that replace a large number of edges with a few strategic ones
to ensure small stretch of distances~\cite{Elkin18tang-emulators}.
% Importantly,
% Such related works aim at designing tractable algorithms to trade-off 
% an increase in the length of the shortest-path between nodes of interest  and the 
% size of the compressed graph.
% between 
% the size of the compressed graph and the  distorsion shortest-path distortion for algorithmic complexity,
% which has led to numerous efforts within the research community.

\isExtended{}{On the other hand,
	lossless compression strategies aim to find compact representations of graphs to be efficiently stored or processed.
	A subset of related work directly deals with communication-efficient re-labeling of nodes
	that enhance graph encoding.
	For example,
	some classical methods exploit algebraic tools such as spectral decomposition of the incidence or adjacency matrix
	that allow encoding the latter with a limited number of codewords,
	while paper~\cite{Kang11icdm-GraphCompressionMining} proposes an algorithm that exploits graph structures such as hubs and spokes.
	A recent survey of lossless compression techniques is given in~\cite{Besta19arXiv-surveyGraphCompression}.
	A different paradigm for lossless compression is based on hypergraphs, which
	%that maintains all original information is based on compact representation of connections
	% by means of hypergraphs.
	generalize standard graphs by allowing hyperedges that connect more than two nodes.
	% in that that
	% while edges connect two nodes,
	% hyperedges are defined by node sets,
	% \ie one hyperedge can connect an arbitrary number of nodes.
	% This feature allows for compactly encode connectivity information.
	Among others,
	paper~\cite{Borici14aina-semanticGraphCompression} tailors semantic data compression,
	\cite{Young21commPhys-hypergraphReconstruction} proposes a procedure to construct hypergraphs from network data,
	\cite{Karypis99desAutConf-multilevelHypergraphPartitioning,Devine06pdps-hypergraphPartitioning} tackle hypergraph partitioning,
	and~\cite{Zhang20iotj-hypergraphSignalProcessing} presents a signal processing framework based on hypergraphs.}

% Slightly departing from general-purpose methods proposed in mathematics and computer science,
Related work in robotics focuses on graph compression to speed up path planning and decision-making.
%  robotic scenarios and applications.
% In particular,
% the recent surge of \dsgname has boosted research on strategies to simplify such huge structures,
% especially in relation to computational efficiency of path planning algorithms.
%For instance,
Silver\setal~\cite{Silver21Number:13-gnnImportanceLearning} use Graph Neural Networks to detect key nodes by
learning heuristic importance scores. Agia\setal~\cite{Agia22corl-Taskography} propose an algorithm that exploits the \dsgname hierarchy
to prune nodes and edges not relevant to the robotic task.
% ranging from navigation to capacity-constrained pick-and-place of objects.
Targeting a related application domain,
Tian\setal~\cite{Tian21ijrr-resourceAwareLoopClosure} 
%are concerned with both 
study
computation and communication efficiency of multi-robot loop closure,
providing a strategy to share a limited number of visual features in multi-robot SLAM,
while Denniston\setal~\cite{Dennison22ral-loopPrioritization} introduce a graph-based method to prune the multi-robot loop closures
in order to save on processing time.
%with theoretical guarantees.
Larsson\setal~\cite{Larsson20tro-qTree,
	Larsson21caadcps-informationTheoreticAbstractions,
	Larsson21ral-infoTheoreticAbstractionsPlanning}
propose algorithms to build hierarchical abstractions of tree-structured representations, for instance enabling 
fast planning on occupancy grid maps at progressively increasing resolution.
% with low computational requirements.
%  of data available on resource-constrained platforms
% to retain relevant information while enabling efficient processing.
% For example,
% paper~\cite{Larsson21ral-infoTheoreticAbstractionsPlanning} illustrates how this algorithm can be used to perform
% path planning on occupancy grid maps at progressively increasing resolution levels,
% with low computational requirements.
			%!TEX root = ../main.tex

\myParagraph{Novel contribution}
In this paper,
we tackle the challenging problem of efficiently sharing \dsgnames for navigation
under hard communication constraints.
We propose two greedy algorithms, \namebu and \nametd (collectively referred to as \emph{\name}), 
that leverage graph spanners to prune nodes and edges from a \dsgname  while minimizing the 
distortion of the shortest paths between locations of interest (\emph{terminal nodes}, see~\cref{fig:cover}).
%with limited communication resources.
% Indeed,
% efficient algorithms running in real time have not been directly addressed in the literature yet.
Compared to the literature,
% we propose novel compression
our algorithms 
(i) are designed to retain navigation-relevant information, %\isExtended{}{ associated with queried tasks},
% provided by a robot requesting the map to be shared,
(ii) leverage the hierarchical structure of the \dsgname for compression,
and (iii) enforce a user-specified size of the compressed \dsgname.
% and on the other hand both spatial information and hierarchical structure of the \dsgname.
%Finally,
%we ground our algorithms in a preexisting
%semantic hierarchy,
%and incorporate task-relevant spatial information in the compression procedure.
Our algorithms are computationally efficient and apply to general \dsgnames.
%
% A similar problem is approached in~\cite{Larsson21caadcps-informationTheoreticAbstractions} via mixed-integer programming,
% hence impractical for real-time tasks,
% with focus on computation efficiency and 
% using an information-theoretic cost with qualitative awareness of the environment geometric.
% available at the queried robot to provide task-effective compression under limited communication resources.
\isExtended{%
	}{Furthermore,
	%We do not limit the scope of compression to a single task,
	we allow for loose specifications of navigation tasks, % executable by a requesting robot
	to make our approach flexible to inexact or uncertain queries:
	for instance, a querying robot requesting a map from another robot may specify a 
	number of potential location it has to navigate between, and this information is used by the 
	queried robot for more effective \dsgname  compression.
	% a queried robot that is asked to share its local map may not precisely know
	% which source (or goal) locations the querying robot is at (or aims to reach),
	% so we allow out methods to consider a set of potential locations for effective compression. 
	}%
\isExtended{%
	}{To meet a sharp budget on transmitted information,
	we design suitable heuristics that exploit a graph spanner of the \dsgname to be sent:
	% these mathematical tools 
	graph spanners allow to trade-off the size of a sub-graph of the \dsgname to be transmitted
	for the maximum distortion suffered by the shortest paths between nodes of interest.
	This helps us design compression algorithms with attention to time performance of navigation tasks,
	for which paths planned on the compressed graph are not much longer compared to paths computed from the original graph
	containing fine-scale spatial information.}%
\isExtended{%
	}{
		
}%
In contrast, related works are either restricted to trees or involve mixed-integer programming~\cite{Larsson21caadcps-informationTheoreticAbstractions,Larsson20tro-qTree}.
\isExtended{%
	}{In particular,
	the approach in~\cite{Larsson20tro-qTree} builds geometric abstractions on-the-fly
	without considering semantic or hierarchical information of the graph to be compressed. }%
Other pruning strategies do not directly target path planning tasks
\isExtended{%
	}{and focus on computational efficiency of local task-planning algorithms}~\cite{Agia22corl-Taskography}.
%that broad range that the queried robot need not precisely know.
Finally,
most works tailored to real-time compression do not allow for hard communication constraints\isExtended{~\cite{Larsson20tro-qTree,Agia22corl-Taskography}}{,
	either turning to soft constraints in the form of Lagrangian-like regularization~\cite{Larsson20tro-qTree},
	or focusing on computational aspects with feasibility requirements~\cite{Agia22corl-Taskography}.
	In particular,
	the work~\cite{Agia22corl-Taskography} proposes to prune the \dsgname to boost efficiency of a local task-planning routine,
	but it does not allow for sharp bounds on the size of the pruned graph,
	and further assumes that a specific task is known beforehand and only needs to be efficiently planned by the robot
	(\eg finding a way to grab and move specified objects)}.
\isExtended{%
	}{
	
	}%
The effectiveness of our algorithms is validated through realistic simulated experiments.
We show that the proposed method meets hard communication constraints
without excessively impacting navigation performance.
%In particular,
%navigation time is not too stretched,
%unless available resources force the robot to transmit extremely coarse information.
For example,
navigation time on the compressed graph increases by at most
8\% %more than planning on the full \dsgname
after compressing the \dsgname to 1.6\% of its size.
% while retaining only 1.6\% of the full \dsgname.

\isExtended{}{
\myParagraph{Paper organization}
In \autoref{sec:problem-formulation}, we present the motivating setup for navigation-oriented compression in the presence of communication constraints,
and states \dsgname compression as an optimization problem
which can be exactly solved via Integer Linear Programming (ILP).
To circumvent computational intractability of the ILP in practice,
we design efficient algorithms that ensure to meet available communication resources
while retaining spatial information useful for navigation.
In particular,
we leverage graph spanners to trad-off size of the compressed graph for distortion of shortest paths:
mathematical background on spanners in provided in~\autoref{sec:spanner},
while explanation of our proposed algorithms is detailed in~\autoref{sec:algorithms}.
In~\autoref{sec:experiments},
we test our approach with realistic simulation software for robotic exploration,
and compare it to the compression approach in~\cite{Larsson20tro-qTree}.
Final remarks and future research directions are given in~\autoref{sec:conclusions}.
}
	%!TEX root = ../main.tex

%\section{\titlecap{communication-aware DSG compression}}
% \section{\titlecap{navigation-oriented \dsgshortname compression}}
% \vspace{4mm}

\section{Navigation-Oriented Scene Graph Compression}
\label{sec:problem-formulation}

\myParagraph{Motivating scenario}
We consider a multi-robot team exploring an unknown environment.
Each robot navigates to gather information and builds a
% a map of the environment as they traverse it.
% In particular,
% they use sensor data to create a 
\dsgname (\dsgshortname) $ \dsg = \lr \nodes[\dsg], \edges[\dsg] \rr $ that describes
the portion of the environment explored so far~\cite{Rosinol20rss-dynamicSceneGraphs,Hughes22rss-hydra,Armeni19iccv-3DsceneGraphs,Wu21cvpr-SceneGraphFusion}.
%~\cite{Armeni19iccv-3DsceneGraphs,Kim19tc-3DsceneGraphs,Talak21neurips-neuralTree}.
%
As robots are scattered across a possibly large area,
they exchange partial maps to collaboratively gather information about the environment.
% as fast and reliably as possible.
In particular,
a robot $ \robotQuery $ may query another robot $ \robotRespond $
to get \mbox{information about the area explored by $\robotRespond$.}\isExtended{}{\footnote{
		We assume robots talk with each other as soon as they get within communication range.}}
% This is useful
%but still unknown to~$\robotQuery$
% for robots to get information about the whole environment without physically navigating all of it.
%for example,
%if $ \robotQuery $ needs to reach some specific location which is nearby the area navigated by $ \robotRespond $.

\myParagraph{Navigation-oriented query}
We assume that the querying robot $\robotQuery$ needs
to reach one or more \textit{target locations} $ \targets \subset \nodes[\dsg] $ within the \dsgshortname $\dsg = \lr \nodes[\dsg], \edges[\dsg] \rr$
\revise{built by} robot $\robotRespond$.
% collected in the target set $ \targets \subset \nodes[\dsg] $,
% which have already been traversed by the queried robot $ \robotRespond $. %but are still unknown to $\robotQuery$.
Such locations, for instance, may be objects or points of interest (\eg the building exits).
% access or connections points of interest,
% such as an exit location.
Hence,
$\robotRespond$ shall transmit its local map
\revise{(\ie nodes and edges of its \dsgshortname)}
such that $\robotQuery$ can reach locations in $ \targets  $
from a set $ \sources \subset \nodes[\dsg] \setminus \targets $ of \textit{source locations}.
In practice,
the latter may \revise{represent} physical access points 
(\eg doors) \revise{at the boundary of} the area explored by $ \robotRespond $
that are near $ \robotQuery $,
and may be %\canOmit{shared by $ \robotQuery $ or} 
estimated by $ \robotRespond $
based on the current location of $\robotQuery$.
In the following,
we generically refer to sources and targets as \textit{terminals} (or \textit{terminal nodes}),
which for the sake of this work are assumed to be \textit{place} nodes in the \dsgshortname.

%\LB{How does $\robotQuery$ know that it should ask $ \robotRespond $ (and not, say, $\gamma$)?
	%	This is worth explaining or mentioning if we assume such knowledge is already given in simulations/experiments.}
%
%\LC{I think that's ok: whenever they are within comms range they talk}

\myParagraph{Communication constraints}
Data sharing among robots occurs over a common wireless channel.
Because of resource constraints of wireless communication,
such as limited bandwidth, robot $\robotRespond$ cannot transmit its entire \dsgshortname to robot $\robotQuery$.
%transmission of the full portion of interest of the \dsgshortname is infeasible in real time.
Specifically,
we assume that robots can send only a small portion of their \dsgshortname each time they receive
a share \isExtended{request.}{request.\footnote{
		Communication constraints can be practically intended as maximum transmission time $ \tmax $:
		a robot first senses the channel and then,
		based on available communication resources,
		estimates the amount of information that can be sent in time $ \tmax $.
		For example,
		assuming bit-rate $ r $,
		specification of $ \tmax $ unambiguously defines the maximum amount of bits $ b_\text{max} = r\tmax $ to be sent,
		which is mapped to a \dsgshortname-related quantity (\eg number of nodes).}}
%Hence,
%the queried robot $\robotRespond$ needs to skim its \dsgname
%to provide useful information while complying with available communication resources.
Hence,
queried robot $ \robotRespond $ needs to compress its \dsgshortname $ \dsg $
into a \reviseral{subgraph} $ \dsgc = \lr \nodesc, \edgesc \rr $, \reviseral{with $\nodesc\subseteq\nodes$ and $\edgesc\subseteq\edges$,}
\reviseral{that contains} at most $B$ nodes (where the budget $B$ reflects the available bandwidth)
%with $ \nodesc \subseteq \nodes $ and $ \edgesc \subseteq \edges $,
\reviseral{in order to comply}
%the new set of nodes $ \nodesc $ 
with communication constraints, %be transmitted in time $ \tmax $ given bit-rate $ \bw $
while \reviseral{at the same time }retaining information useful for robot $ \robotQuery $ to navigate between the terminal nodes.

\myParagraph{Pruning \dsgnames}
Assuming navigation-oriented queries,
the relevant information reduces to nodes and edges describing efficient paths robot $\robotQuery$ can use to move across the map.
Specifically,
the collection of all shortest paths between a source $ \source\in\sources $ and target $ \target\in\targets $
is the least information ensuring that navigation by $ \robotQuery $ takes the shortest possible time,
\ie the time a robot with full knowledge of the map would take.
%that $ \robotRespond $ would employ having full knowledge of the map.
%Indeed,
%if shortest paths between $ \source $ and $ \target $ are partially unknown to $\robotQuery$,
%the latter will traverse different locations to reach $ \target $ from $ \source $,
%and hence navigation will take more time than if the shortest path were followed.

\newcommand{\mpwfour}{4.0cm}
\begin{figure}[t]
	\begin{center}
		\begin{minipage}{\textwidth}
			\begin{tabular}{cccc}
				\begin{minipage}{\mpwfour}%
					\centering%
					\includegraphics[width=0.8\columnwidth]{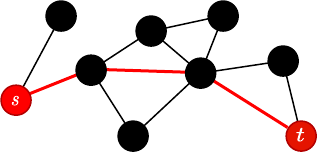}
					\\
					{\small (a) Full graph (length $ 3 $).}
				\end{minipage}
				&
				\begin{minipage}{\mpwfour}%
					\centering%
					\includegraphics[width=0.8\columnwidth]{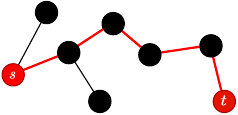}
					\\
					{\small (b) Pruned graph (length $ 5 $).}
				\end{minipage}
			\end{tabular}
		\end{minipage} 
		\caption{Distortion of shortest path from $ \source $ to $ \target $ (thick red).}
		\label{fig:shortest-paths}
	\end{center}
\end{figure}

However,
transmitting all nodes in the shortest paths may violate the communication constraint (see \cref{fig:shortest_path_compare}):
this can happen with many terminals or if shortest paths have little overlap.
Hence,
heavier pruning of the \dsgshortname might be needed to make communication feasible.
This means that information useful for path planning will be partially unavailable to the querying robot's planner.
In other words,
because the \dsgshortname $ \dsg $
cannot be fully sent,
the distance (length of a shortest path) between a pair of terminals in the transmitted graph $ \dsgc $
will be larger than the distance between those same terminals in the original \dsgshortname.
A schematic example is provided in~\autoref{fig:shortest-paths},
where the length of the shortest path between nodes $ \source $ and $ \target $ increases from $ 3 $ to $ 5 $
after node and edge removal.
For example,
a robot may prune place nodes within a room,
or share only the room node as a coarse representation of places.
This requires less communication,
but the querying robot $\robotQuery$,
which receives a coarser map,
will be forced to,
\eg take a longer detour across a room, 
% pass through the room node
instead of traversing the original shortest path along a set of place nodes.
%  along place nodes.
% introducing time-consuming exploration because of lack of information about free space
% (which is encoded by edges in the place layer).
Mathematically,
this means $ \dist{\dsgc}{s}{t} \ge \dist{\dsg}{s}{t} $
for any $ \source\in\sources $
and for any $ \target\in\targets $,
where $ \dist{\mathcal{G}}{u}{v} $ is the distance
from node $ u $ to node $ v $ in $ \mathcal{G} $.

\myParagraph{Problem formulation}
For the querying robot $\robotQuery$ to navigate efficiently, 
% In the interest of navigation time and energy consumption,
%it is desirable that
the distance $ \dist{\dsgc}{s}{t} $ between $ \source\in\sources $ and $ \target\in\targets $ in the transmitted graph $ \dsgc $
should not be much larger than the distance in the original graph $ \dsg $.
Hence,
the queried robot $\robotRespond$ shall prune $ \dsg $ so as to minimize the \textit{distortion},
or \textit{stretch},
between shortest paths in the original and compressed graphs, while meeting the communication budget $B$.
%Assuming additive distortion,
This can be cast into the following optimization problem:
\begin{mini!}
	{\hspace{-3mm}\dsgc \subseteq \dsg}{\distortion}
	{\label{eq:optimization-problem}}
	{\label{eq:optimization-problem-objective-cost}}
	\addConstraint{\hspace{-3mm}\dist{\dsgc}{s}{t}}{\le \dist{\dsg}{s}{t} + \distortion \maxweight{\dsg}{s}{t}\;\,}{\forall (s,t)\in\pairs}\label{eq:optimization-problem-distortion-constr}
	\addConstraint{\hspace{-3mm}|\nodes[\dsgc]|}{\le\bw,}{}\label{eq:optimization-problem-bw-constr}
\end{mini!}
where $ \maxweight{\mathcal{G}}{u}{v} $ is the maximum edge weight on a shortest path from $ u $ to $ v $ in $ \mathcal{G} $
and $ \pairs \subseteq \sources\times\targets $ is the set of considered source-target pairs.
Constraint~\eqref{eq:optimization-problem-bw-constr} ensures that the amount of transmitted information
(number of nodes)
meets the communication constraint,
while constraint~\eqref{eq:optimization-problem-distortion-constr} and cost~\eqref{eq:optimization-problem-objective-cost}
encode minimization of the maximum distortion incurred by the shortest paths.
The coefficient $ \maxweight{\dsg}{s}{t} $ in~\eqref{eq:optimization-problem-distortion-constr}
makes the distortion computation meaningful for weighted graphs.

Problem~\eqref{eq:optimization-problem} can be solved by means of integer linear programming (ILP),
see \isExtended{the technical report~\cite[Appendix~A]{Chang22arxiv-dlite}}{\autoref{app:ilp}}.
However,
the runtime complexity of ILP solvers is subject to combinatorial explosion, % of exact search methods,
making this approach impractical for online operation.
Hence,
we propose greedy algorithms that require lighter-weight computation,
based on \textit{graph spanners}.
\isExtended{Background about these tools is given in~\cite[Section~III]{Chang22arxiv-dlite}.}{}
	\isExtended{}{%!TEX root = ../main.tex

\section{\titlecap{background: graph spanners}}
\label{sec:spanner}

We ground our compression algorithm in the concept of \textit{graph spanner}~\citep{Ahmed20csr-GraphSpannersReview}.
In words,
a spanner is a compressed (\ie sparse) representation of a graph 
such that shortest paths between nodes are distorted at most by a user-defined stretch.
Formally,
a spanner $ \mathcal{G}'=\lr \nodes,\edges'\rr $ of graph $ \mathcal{G}=\lr\nodes,\edges\rr $ is a subgraph
such that $ \edges'\subseteq\edges $ and the following inequality holds for $ u,v\in\nodes $,
\begin{equation}\label{eq:distortion-spanner}
	\dist{\mathcal{G}'}{u}{v} \le \alpha\dist{\mathcal{G}}{u}{v} + \beta\maxweight{\mathcal{G}}{u}{v},
\end{equation}
where $ \alpha \ge 1 $ and $ \beta \ge 0 $ are given constants.
For generic $ \alpha $ and $ \beta $,
$ \mathcal{G}' $ is called an $ (\alpha,\beta) $-spanner,
whereas if $ \beta $ (resp. $\alpha$) is equal to zero (resp. one)
$\mathcal{G}'$ is called \textit{$\alpha$-multiplicative} spanner
(resp. $\beta$-\textit{additive} or $ +\beta $ spanner).
Inequality~\eqref{eq:distortion-spanner} may hold for all nodes in $\mathcal{G}$ or 
for a few pairs as in~\eqref{eq:optimization-problem-distortion-constr}:
in the latter case,
the resulting subgraph is referred to as a \textit{pairwise spanner}.

Applications of spanners include navigation or packet routing in large graphs,
whose size makes running path planning algorithms in the original graph computationally infeasible,
%and a sparse representation has to be used instead
\cite{Dobson14ijrr-roadmapSpanners,Klein06tc-subsetSpanner}.
In this case, one can compute a spanner of the original graph and run planning 
algorithms on the spanner instead.
% a spanner is then typically built in the first place and used for,
% \eg path planning or routing in place of the original graph.

As one can see from~\eqref{eq:distortion-spanner},
the characterization of spanners shares similarity with problem~\eqref{eq:optimization-problem}.
Unfortunately,
no method is known in the literature to build a spanner given a fixed node (or edge) budget,
whereas algorithms usually enforce stretch~\eqref{eq:distortion-spanner} given %the interested pairs of nodes
input parameters $ \alpha $ and $\beta$
while attempting to minimize the total spanner edge-weight to obtain lightweight representations.
The standard formulation of the graph spanner problem can be then written as follows~\cite[Problem 2]{Ahmed20csr-GraphSpannersReview},
\begin{mini!}
	{\edgesc \subseteq \edges}{\sum_{(i,j)\in\edgesc}\weight{\dsg}{i}{j}}
	{\label{eq:spanner-problem}}
	{\label{eq:spanner-problem-objective-cost}}
	\addConstraint{\dist{\dsgc}{s}{t}}{\le \alpha\dist{\dsg}{s}{t} + \distortion \maxweight{\dsg}{s}{t},}{}\label{eq:spanner-problem-distortion-constr}
\end{mini!}
where $ \weight{\dsg}{i}{j} $ is the weight of edge $ (i,j) $
and the objective function for unweighted graphs reduces to counting the number of edges.
For multiplicative spanners
this problem was quickly solved,
with the classical work~\citep{Althofer93dcg-sparseSpanners} proposing and analyzing a greedy algorithm
which is known to be the best (in terms of spanner size) that runs in polynomial time.
Additive and  $(\alpha,\beta) $-spanners are instead more complex to build,
and many algorithms have been proposed in the literature:
early efforts were devoted to unweighted
graphs~\citep{Kavitha15stacs-pairwiseSpanners,
	Baswana10ta-spanners,
	Abboud17acmj-additiveSpannerBound,
	Cygan13stacs-pairwiseSpanners},
while subsequent work has focused on the general weighted
case~\citep{Elkin22dc-WeightedAdditiveSpanner,
	Elkin22swat-almostShortestPath,
	Ahmed20gctcs-weightedAdditiveSpanners,
	Ahmed21sea-multiLevelSpanners}.
Other studies are concerned with
distributed~\citep{Censor-Hillel18dc-distributedAdditiveSpanners} and dynamical~\citep{Baswana08da-spannerDynamicAlg} methods,
Euclidean graphs~\citep{Arya95tc-EuclideanSpanners},
and reachability preservation in digraphs~\citep{Abboud18soda-ReachabilityPreservers},
to mention a few.

To the best of our knowledge,
the only paper to address the presence of an edge budget $ E_\text{max} $ is~\cite{Kobayashi20stacs-fptAlgSpanner}.
However,
the algorithm  proposed in~\cite{Kobayashi20stacs-fptAlgSpanner} receives in input also parameters $ \alpha $ and $ \beta $,
and checks feasibility of an $ (\alpha,\beta) $-spanner with at most $ E_\text{max} $ edges.
Furthermore,
its runtime increases exponentially with $ E_\text{max} $,
making it unsuitable for robotics applications.

A possible way to tackle the problem at hand
is to iteratively build spanners with larger and larger distortion,
until the budget is met.
However,
several issues can hamper such a strategy.
First,
running a spanner-building algorithm several times may be time-consuming.
Second,
while small-sized (\ie with $ O(n^{1+\eps(\alpha)}) $ edges,
for some small $ \eps(\alpha)>0 $) multiplicative spanners can be built for any given constant coefficient $ \alpha \ge 1 $,
few constant-distortion additive spanner constructions are known for weighted graphs,
with coefficient $ \beta \in \{2,4,6\} $.
Conversely,
\textit{polynomial} distortion $ \beta = \beta(n) $ is needed
to build additive spanners with near-linear size~\cite{Abboud17acmj-additiveSpannerBound},
thus the trade-off between spanner size and path distortion is not easy to exploit.

An important point is that multiplicative and additive distortions may yield
dramatic differences in paths induced by the spanner.
In particular,
multiplying path length by a constant factor in large graphs may be undesirable in practice:
for example,
if a navigation task nominally takes one hour,
stretching it to two or three hours yields substantial performance degradation.
Conversely,
additive stretch is usually preferred because it provides a constant time overhead,
which is why we used this kind of distortion in our problem formulation.

In the following,
we illustrate a heuristic procedure that allows us to meet the budget constraint in~\eqref{eq:optimization-problem-bw-constr},
runs in real time,
and enforces a low distortion of shortest paths
as measured by condition~\eqref{eq:optimization-problem-distortion-constr}.}
	%!TEX root = ../main.tex

\section{\titlecap{3D scene graph compression algorithms}}
\label{sec:algorithms}

% Given the practical difficulties mentioned above,
We propose \textit{\name}, %(\namefull),
%two greedy algorithms 
a compression method for \dsgshortnames to meet communication constraints
with attention to navigation efficiency.
We design two versions of \name,
which %leverage 
are initialized with a spanner of the full \dsgshortname (\autoref{sec:alg-build-spanner})
% built during initialization with given input distortion parameters
and tackle the compression problem from opposite perspectives.

The first algorithm,
\namebu
%presented in
(\autoref{sec:alg-bottom-up}),
performs progressive bottom-up compression
of the spanner computed during initialization,
exploiting the \dsgshortname abstraction hierarchy.
In contrast,
the second algorithm,
\nametd
%presented in~
(\autoref{sec:alg-top-down}),
works top-down expanding nodes with the spanner as a target.
			%!TEX root = ../main.tex

\subsection{Intuition and the Role of the \dsgname Hierarchy\!\!\!}
\label{sec:compression-idea}

\isExtended{}{
	Ideally, navigation-oriented compression of a 
	% compressing the 
	\dsgshortname 
	%with attention to navigation efficiency
	would require searching among all possible subgraphs of $\dsg$
	to find one that minimally stretches paths between terminals.
	Such a search is prone to combinatorial blow-up
	and is thus impractical.
}

Assume we want to design a greedy procedure that removes nodes and edges in $\dsg$ % \isExtended{}{contained in the \dsgshortname}%
while limiting the incurred path stretch.
\isExtended{}{This goal is subject to a nontrivial trade-off. 
	On the one hand,
	to ensure low stretch (\ie retain navigation performance),
	it is desirable to parse one or a few nodes at each iteration
	so as to introduce extra distortion in a %fine-tuned,
	controlled way.
	On the other hand,
	parsing too few nodes at each time induces a large number of total iterations,
	and is computationally expensive for online operation.
	%cause the overall loop to last for too long.
	Hence,
	an effective algorithm should effectively choose the size of node batch to be greedily compressed at each iteration
	to strike a balance between compression quality and runtime.
	
}%
%Towards this goal,
To this aim,
we crucially exploit the \textit{hierarchical structure} of the \dsgshortname.
\reviseral{We refer to a node $m$ that is adjacent to node $n$ in the upper layer as a \textit{child} of $n$
	to stress the hierarchical semantics of the \dsgshortname,
	and symmetrically call node $n$ the \textit{parent} of $m$.
	The children of $n$ in graph $\dsg$ are denoted by $ \children[\dsg](n) $. %\footnote{
%		\reviseral{By construction,
%		every node has one parent (but possibly many children).}
%	}
	Also,
	the set $\edges[\dsg](n)$ gathers all edges incident to $n$ in $\dsg$.
}%
The \dsgshortname hierarchy allows us to see a node %of the \dsgshortname $ \dsg $ 
as a 
``compressed'',
or ``abstract'',
representation of its children. % $ \nodes[\dsg](n) $ in the layer below. %$ \ell+1 $:
Hence,
transmitting $ n $ rather than $ \children[\dsg](n) $ saves communication 
and conveys partial spatial information about nodes in $ \children[\dsg](n) $. %$ \nodes[\dsg](n_\ell) $.
For instance,
let $ \children[\dsg](n) $ represent places inside a room
and $ n $ the associated room node.
% (which includes information about the room centroid and bounding box~\cite{Hughes22rss-hydra,Rosinol21ijrr-Kimera}).
%\eg geometrically described by the coordinates of the room centroid.
A robot that needs to reach a location $ \target\in\children[\dsg](n) $ 
(\eg the door) in that room \revise{with no bandwidth constraints} 
can be provided with a sequence of places to reach $t$.
% However, as a compressed representation, 
Alternatively, the robot can  be given the room node $ n $
\revise{and it needs to} explore the room to find the target $t$:
% Then, 
% if 
% it can first approach the room center $ n_\ell $ %through a path planned on the received map,
% and then explore the area until it finds the target.
this extra exploration 
\isExtended{}{,
	(corresponding to additional path stretch in the compressed \dsgshortname) }%
\reviseral{takes longer,} 
degrading navigation performance,
but allows for compression to meet communication constraints.
%Mathematically,
The navigation time for local exploration
\isExtended{%
}{(\eg to reach a place from the room centroid) }%
is encoded by the weights of \reviseral{edges connecting non-finest resolution nodes or nodes at different resolutions (layers).
	For our experiments,
	we derive such weights from the full \dsgshortname as detailed in \isExtended{\cite[Appendix~B]{Chang22arxiv-dlite}.}{\autoref{app:edge-weight}.}
	However,
	we argue that a robot can estimate all weights on-the-fly (while building the \dsgshortname)
	based on the actual navigation time it experiences.
%	The same holds for edges connecting nodes in the same layer with no direct path
%	(\eg to move from a room to another room without knowledge of intermediate place nodes).
%	Details about the calculation of such weights,
%	which 
}

The discussion above suggests a simple way to compress the \dsgshortname:%
% From the above discussion,
% abstraction levels in the \dsgshortname suggest a way to tackle the trade-off on the number of nodes to be iteratively "compressed":
\isExtended{%
}{
a greedy procedure can be devised so that
}
\revise{nodes in a layer can be progressively replaced by their parent nodes in the layer above.
% considerably reducing communication resources.
Every time we replace nodes with more ``abstract'' ones (rooms, buildings)
the length of the paths passing through those nodes increases,
\reviseral{indicating longer navigation}.
Hence,
we can opportunistically select which nodes to ``abstract away''
so as to achieve a small stretch in the paths between terminals. 
In alternative, we can start with a coarse representation (including only the highest abstraction level) and 
expand it to reduce the stretch of the paths.}
We present these two greedy strategies below and initialize both procedures by computing a spanner of the given \dsgshortname, 
as explained next.
% In particular,
% spatial information encoded in the place layer of the \dsgshortname can be used to approximately quantify
% the navigation stretch introduced each time an "abstract" node (\eg room or building) is traversed by a path
% instead of nodes at lower,
% spatially finer-scaled layers.
% For example,
% if multiple place nodes are replaced by a room node,
% the navigation time for the path will take into account additional time needed to explore (via local navigation)
% that room without knowing free-space edges of the place layer.
% Moreover,
% in this way the maximum number of compression iterations is typically small,
% because nodes in layers above place nodes are much fewer than place nodes.
			%!TEX root = ../main.tex

% \subsection{\titlecap{building \dsgshortname spanner}}
\subsection{Building a \dsgshortname Spanner}
\label{sec:alg-build-spanner}

\isExtended{}{Both proposed algorithms build a spanner of the \dsgshortname during initialization.
	A detailed description of how each procedure uses this spanner is deferred to~\cref{sec:alg-bottom-up,sec:alg-top-down}.}
%the dedicated sections.
%In this section,
%we outline the algorithm used to find a DSG spanner in our implementation.

\isExtended{The literature provides several algorithms to produce spanners of an input graph
	given a user-specified stretch on the distance between terminals.
	The spanner need not meet our budget constraint, hence we use it just as initialization for \name.
	We adapt the algorithm in~\cite[Section 5]{Elkin22dc-WeightedAdditiveSpanner}
	\revise{to build} a spanner of the \reviseral{full} \dsgshortname with additive path stretch. % tuned by user-defined parameters.
%	\marginnote{\LC{the next 5 lines are very high-level: 
%		can we say if the random selection is about the edges or nodes?
%	how do we ``retain only nodes and edges %on shortest paths between terminals,
%			which are relevant for navigation''? }}
	\revise{The procedure \reviseral{initializes the spanner} with a random selection of edges:}
	to exploit the \dsgshortname hierarchy,
%	which helps shrinking the produced spanner,
	\revise{we modify the original algorithm by manually adding cross-layer edges during the initialization.}%ensure that are selected % connecting places nodes to upper layers,
	\revise{Also,
		once the spanner is built,
		we retain only nodes and edges %on shortest paths between terminals,
		relevant for navigation
		by removing all nodes that are not traversed by shortest paths between terminals \reviseral{in the spanner just built}.\footnote{
		\reviseral{We assume that the robot can compute shortest paths between terminals in a reasonable time as compared to the overall compression procedure.}}}
		This greatly reduces the graph to be compressed,
	%	retaining only nodes and edges which are relevant for navigation,
		making our compression strategies \reviseral{based on hierarchal abstractions} more efficient.
		We call this subroutine \texttt{build\_spanner}.
		In the interest of space,
		we defer more details to~\cite[Section~III-B]{Chang22arxiv-dlite}.
%	\begin{rem}
%		The \dsgshortname spanner exploits the overlap among paths connecting terminals
%		to keep a few key edges and nodes.
%		On the other hand,
%		keeping all shortest paths on the original \dsgshortname
%		does not exploit the hierarchy of the graph
%		and does not enforce a maximum size of the
%		compressed graph.
%	\end{rem}
}{
	\cref{alg:build-spanner} describes how to build a spanner of the \dsgshortname
	%given input parameters 
	that enforces a user-defined maximum additive stretch for distances between specified terminal pairs in $ \pairs $.
	We adapt our algorithm from~\cite[Section 5]{Elkin22dc-WeightedAdditiveSpanner}.
	Specifically,
	the procedure~\cite{Elkin22dc-WeightedAdditiveSpanner} can trade spanner size for stretch according to input parameters,
	building a $ +cn^{\frac{1-\eps}{2}}\alpha W_\text{max}^\dsg $ spanner
	of size $ O(n^{1+\eps}) $.\footnote{
		Parameter $ \alpha $ might depend on $ n $,
		\eg the authors in~\cite{Elkin22dc-WeightedAdditiveSpanner} use $ \alpha(n) = \log n $.}
	That algorithm is intended for generic spanners (not pairwise),
	hence
	we adapt it to our scope by retaining only nodes and edges needed to connect terminal pairs in $ \pairs $,
	and deleting all others.
	\cref{alg:build-spanner} is composed of two sequential stages:
	an initialization phase that builds a temporary spanner $ \dsgc' $
	with a small number of edges
	that attempts to keep low initial path distortions,
	and a "buying" phase where edges are added to meet the stretch constraint.
	The initialization selects edges in three ways:
	performing a $ d $-light initialization~\cite[Section 2]{Ahmed20gctcs-weightedAdditiveSpanners} with appropriate $ d $ (\cref{alg0:light-initialization}),
	which in words ensures that each node has some initial neighbors;
	randomly picking cross-layer edges to exploit the \dsgshortname hierarchy (\cref{alg0:random-vertical-edges});
	adding a greedy multiplicative spanner~\cite[Section 2]{Althofer93dcg-sparseSpanners} to reduce large path distortions (\cref{alg0:greedy-spanner}).
	Then,
	for each pair $ (s,t) $
	the shortest path $ \path{\dsgc'}{\source}{\target} $ from source $ \source $ to target $ \target $
	in the temporary spanner $ \dsgc' $ is considered (in suitable order):
	in case the stretch in $ \dsgc' $ exceeds the constraint,
	edges and nodes from a shortest path in the original graph $ \dsg $ are added to both $ \dsgc' $ and the final spanner $ \dsgc $
	(Lines~\ref{alg0:add-edges-1}--\ref{alg0:add-edges-2}),
	otherwise,
	the shortest path in $ \dsgc' $ is directly added to the final spanner $ \dsgc $ (\cref{alg0:add-edges-3}).\footnote{
		\reviseral{We assume that the robot can compute shortest paths between terminals in a reasonable time as compared to the overall compression procedure.}}
	We refer to this subroutine as \texttt{build\_spanner}.}

\isExtended{}{
	\begin{algorithm}[t]
		\caption{Build spanner}
		\label{alg:build-spanner}
		\DontPrintSemicolon
		\KwIn{\dsgshortname $ \dsg $, 
		terminal pairs $ \pairs $, 
		user parameters $ \eps>0, p\in[0,1], \alpha > 2, c > 0 $.}
		\KwOut{\dsgshortname spanner $ \dsgc $.}
		$ \dsgc_1 \leftarrow n^{\eps} $-light initialization of $ \dsg $;\;\label{alg0:light-initialization}
		$ \dsgc_2 \leftarrow $ random sample of cross-layer edges of $ \dsg $ w.p. $ p $;\;\label{alg0:random-vertical-edges}
		$ \dsgc_3 \leftarrow $ $ \alpha $-multiplicative spanner of $ \dsg $;\;\label{alg0:greedy-spanner}
		$ \dsgc' \leftarrow \dsgc_1 \cup \dsgc_2 \cup \dsgc_3 $; \tcp*{to compute paths}
		$ \dsgc \leftarrow \pairs $;\;
		\ForEach(\tcp*[f]{sorted by $ \maxweight{\dsg}{\source}{\target} $}){$ (\source,\target) \in \pairs $}
		{
			\If{$ \dist{\dsgc'}{\source}{\target} > \dist{\dsg}{\source}{\target} + cn^{\frac{1-\eps}{2}}\alpha\maxweight{\dsg}{\source}{\target} $}
			{
				$ \dsgc' \leftarrow \dsgc' \cup \path{\dsg}{\source}{\target} $;\;\label{alg0:add-edges-1}
				$ \dsgc \leftarrow \dsgc \cup \path{\dsg}{\source}{\target} $;\;\label{alg0:add-edges-2}
			}
			\Else %\Comment{add to final spanner}
			{
				$ \dsgc \leftarrow \dsgc \cup \path{\dsgc'}{\source}{\target} $;\;\label{alg0:add-edges-3}
			}
		}
		\Return $ \dsgc $.
	\end{algorithm}
}

\isExtended{}{Before diving into the core of our compression algorithms,
	it is worth reinforcing the motivation to use graph spanners.
	% Indeed,
	% building a spanner greatly reduces the graph to be compressed up front,
	% retaining only nodes and edges which are both  relevant for navigation
	% and sharply enhance efficiency of our proposed compression strategies.
	\cref{alg:build-spanner} outputs a spanner with given maximal stretch of shortest paths, but
	does not guarantee that the produced spanner matches the desired node budget $ \bw $.
	% but its size can only be upper bounded and in practice
	% further processing is needed to meet communication constraints,
	% so that both (exact) spanner distortion and (estimated) size will change.
	As noted in~\autoref{sec:spanner},
	we cannot straightly apply a state-of-the-art spanner construction 
	because no real-time algorithm in the literature addresses the presence of an exact budget.
	However,
	building a spanner greatly reduces the graph to be compressed up front,
	retaining only nodes and edges which both are relevant for navigation
	and sharply enhance efficiency of our proposed compression strategies.
	% Instead,
	% parsing the full graph from scratch would be more inefficient and time consuming.
	Moreover,
	even though path distortion may be increased to satisfy communication requirements,
	the user-defined stretch guaranteed by the spanner algorithm allows us to start from a maximum desired distortion:
	% which is quantitatively related to the size bound:
	hence,
	if the latter is chosen loose enough,
	we may expect that the spanner output by~\cref{alg:build-spanner} is already somewhat close to the communication constraint,
	so that additional distortion will not be very high.
	In particular,
	the spanner construction leverages overlapping portions of paths to select a handful of key edges and nodes,
	whereas other navigation-efficient constructions,
	such as the collection of all shortest paths,
	do not exploit the graph structure to enhance compression.
	Furthermore,
	there are no tight bounds for the size of shortest paths, %(while this is a notable feature of spanner algorithms).
	hence one cannot predict how much paths will be stretched in order to meet the communication budget.
}
			%!TEX root = ../main.tex

%\subsection{\titlecap{bottom-up compression algorithm}}
\subsection{\titlecap{BUD-Lite: a bottom-up compression algorithm}}
\label{sec:alg-bottom-up}

The idea behind our first algorithm (\namebu, short for Bottom-Up D-Lite) is to
iteratively \emph{compress}
the \dsgshortname spanner produced by\isExtended{ \texttt{build\_spanner}}{~\cref{alg:build-spanner}}.
% going up along the \dsgshortname hierarchy,
%by abstracting them away
%according to the discussion in~\autoref{sec:compression-idea}.
The mechanism is simple:
we progressively replace batches of nodes with their parents to reduce size,
%(\eg removing places nodes and adding the associated room node),
while attempting to keep the stretch incurred by the shortest paths between terminals low.
%\canOmit{To enhance compression granularity,
%we consider one terminal pair $ (\source,\target) $ at a time
%and %compress a stretch of nodes along the shortest path,
%replace a stretch of nodes along the shortest path with an abstract representation.}

To gain intuition,
consider~\autoref{fig:bottom-up}
that illustrates three steps of \namebu
on a toy \dsgshortname.\footnote{
	\revise{While~\autoref{fig:bottom-up} considers only place and room layers for the sake of visualization,
	our algorithm applies to \dsgshortnames with any number of layers.}
} % composed of place and room nodes.
Dashed edges and light-colored nodes are part of the full \dsgshortname $ \dsg $
and can be added to the compressed \dsgshortname $ \dsgc $.
The latter is marked with solid lines and brighter colors.
%The initial spanner (top left) contains nodes in the place layer,
%thus it retains precise spatial information but does not meet the communication budget.
The first iteration of \namebu parses the path from $ \source $ to $ \target_1 $
and \reviseral{abstracts away %the first stretch of places.
%Specifically,
place nodes $ P_3 $ and $ P_4 $,
which are replaced} with room node $ R_2 $ \reviseral{that} is a coarse representation of those places (top right).
Room $ R_1 $ is \reviseral{skipped} because \reviseral{it does not reduce budget} as compared to keeping $ P_1 $.
%Also,
Place $ P_2 $ is not removed yet because it lies also on the path connecting pair $ (\source,\target_2) $,
\reviseral{while $P_5$ is still needed to connect $(\source,\target_1)$.}
Node $ P_2 $ is removed at the second round
when the path from $ \source $ to $ \target_2 $ is parsed and shortcut through place node $ P_1 $ and room node $ R_2 $ (bottom right).
The final step parses the last portion of the path connecting $ (\source,\target_1) $
and abstracts away the remaining places $ P_5 $ and $ P_6 $ under rooms $ R_2 $ and $ R_3 $ (bottom left).
An example on an actual \dsgshortname build from simulated data
is shown in~\autoref{fig:bottom-up-example},
where the room node is used to abstract several places.
\isExtended{}{More results on \dsgshortnames from simulated navigation data are provided in~\autoref{app:experiments}.}
% To ground the discussion,
% we next detail the steps of the proposed procedure.

\begin{figure}
	\centering
	\includegraphics[width=.8\linewidth]{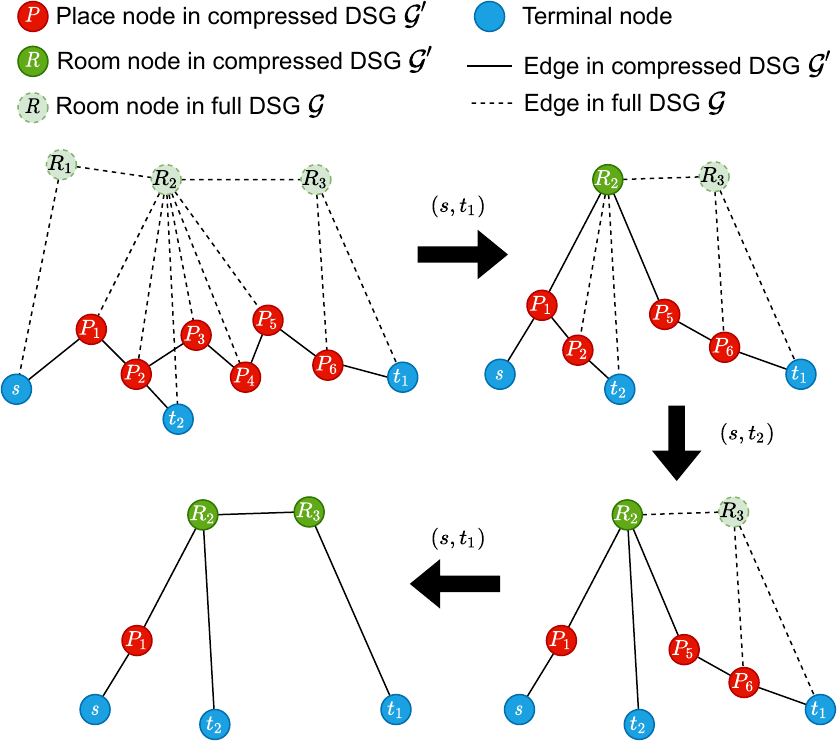}
	\caption{Illustration of the \namebu procedure with source $ \source $
		and targets $ \target_1, \target_2 $.
		At each iteration,
%		a terminal pair is considered and
		place nodes in a shortest path between terminals are replaced by a room node.
		\isExtended{}{Nodes are removed when none of the terminal pairs $ (\source,\target_1) $ and
		$ (\source,\target_2) $ connects through them.
		Note that the final graph cannot be further pruned without disconnecting terminals.}
	}
	\label{fig:bottom-up}
\end{figure}

We formally introduce the compression procedure in \cref{alg:bottom-up}.
The compressed graph is \revise{initialized} as the
\dsgshortname spanner $ \dsgc $ output by\isExtended{ \texttt{build\_spanner}}{~\cref{alg:build-spanner}} (\cref{alg1:initialization}).
\revise{In the following,
the symbol $\layer{\dsg}{i}$ refers to \reviseral{the nodes within the $i$th layer of $\dsg$.}
%the subgraph of the \dsgshortname $\dsg$ that composes its $i$th layer,
%that is,
%nodes and edges in that layer.
For example,
\reviseral{$\layer{\dsg}{0}$ collects all place nodes of the full \dsgshortname $\dsg$.}}
%$\layer{\dsg}{0}$ is the place layer of the original \dsgshortname $\dsg$ with all place nodes and edges connecting them.
Also,
$\path{\dsg}{\source}{\target}$ denotes the shortest path between $\source$ and $\target$ in $\dsg$.
\revise{The external loop at~\cref{alg1:foreach-layer} parses each layer $ \layer{\dsgc}{i} $ of $ \dsgc $,
%each layer $ \ell $ 
%from source $ \source $ to target $ \target $
%one path at a time 
%(\cref{alg1:foreach-path}),
%\LC{it would be nice to have the bottom layer be layer $0$}
%considering first nodes in 
starting from the bottom ($ i = 0 $) %of $ \dsgc $ 
%and progressively turning to spatially coarser layers (\cref{alg1:foreach-layer},
and moving to the upper layer $\layer{\dsgc}{i+1}$ after $\layer{\dsgc}{i}$ has been compressed (\cref{alg1:empty-layer}).}
%to the top layer ($ \ell = L $).
At each iteration of the inner loop at~\cref{alg1:foreach-path},
the algorithm checks if the shortest path connecting terminals $ \source $ and $ \target $ in $\dsgc$
\reviseral{is traversed by nodes} in layer $\layer{\dsgc}{i}$
%	(denoted by $ \nodes[\dsgc](n) \doteq \path{\dsgc}{\source}{\target}\cap\nodes[\dsgc](n)$) %and named \emph{\red{$(\ell-1)$-stretch}},
with the same parent node $ n\in\layer{\dsg}{i+1} $
(\cref{alg1:check-children}):
%in this case,
if \reviseral{this is the case},
\reviseral{such nodes with common parent}
%all nodes in $ \nodes[\dsgc](n) $
% then,
% this node batch $ \nodes[\dsg](n_\ell) $ is 
are removed from $ \path{\dsgc}{\source}{\target} $ 
and replaced (\textit{compressed}) with their parent node $ n $
(\cref{alg1:replace-nodes-in-path}).\footnote{
	For consistency of navigation,
	we do not compress terminal nodes in our implementation,
	but this can be changed to accommodate the budget constraint.
}
%For example,
%consecutive place nodes belonging to the same room
%are replaced by the corresponding room node,
%one path at a time,
%one room at a time.
%For example,
%this step may replace a sequence of places nodes with the corresponding room node.
%\LC{a bit unclear why this is proportional:}
Such a \textit{compression} in the graph
causes a corresponding \textit{stretch} of the actual path followed by the robot,\isExtended{}{~the amount of which depends on both the involved layer $ \layer{\dsgc}{i} $ and the number of compressed nodes,}
in light of the discussion in~\autoref{sec:compression-idea}.
%Nevertheless,
The nested structure of~\cref{alg:bottom-up} looping over layers externally (\cref{alg1:foreach-layer})
and over paths internally (\cref{alg1:foreach-path})%
\isExtended{}{~adds one abstraction level at a time for each path 
	(the layer $ \layer{\dsgc}{i} $ is fixed in loop~\cref{alg1:foreach-path}),
	and hence}
%\reviseral{\red{LC(unclear):adds one coarse layer}}
stretches distances in a balanced fashion:
%\reviseral{this is because	the innermost loop adds (at most) one coarse node to every path at each iteration,}
%so that all paths are expected to incur comparable distortion eventually.
\isExtended{%
%	This is caused by the designed nested iterative process: 
	after the inner loop parses all paths once,
	each path traversed by nodes in the finest layer is compressed by one coarse node
	(\eg a room node replacing place nodes),
%	Hence, 
%	after one outer iteration, 
%	all paths are reduced by one node 
	and such coarse nodes are in the same layer for all compressed paths
	(\eg during the first iteration of the outer loop,
	places can be compressed to rooms but not to buildings). 
	This means that the additional stretch is the same for all paths,
	on average: differences arise if the paths traverse unbalanced locations,
	\eg a path passes through a large room (which yields high distortion) while one through a small room (low distortion).
}{for example,
	if paths are made of places nodes,
	the first iteration of the inner loop
	compresses only one room for each path,
	so that at no point during the compression procedure a path
	is overly compressed with respect to other paths
	(\ie	it is not possible that a path is entirely abstracted to room nodes
	while another is kept with all places nodes).
	In general,
	this allows fine-scale spatial information to be retained as long as possible,
	and coarser layers (\eg buildings) to be used only after finer layers (\eg rooms) have been fully exploited for all paths
	(\ie
	building nodes can be used only if all paths contain room nodes and no places nodes,
	possibly except for terminal nodes).
}%
%Furthermore, %\LC{what is ``path consistency''?}
To ensure paths are always feasible,
\isExtended{nodes are}{we use a data structure $ \dict $ to track which paths are using nodes in $ \dsgc $ (\cref{alg1:build-dict}):
	only when a node is traversed by no path (Lines~\ref{alg1:delete-node-start} to~\ref{alg1:delete-node-end}),
	it is} removed\isExtended{ when they are unused \reviseral{(Lines~\ref{alg1:delete-node-start} to~\ref{alg1:delete-node-end})}.}{.}
\reviseral{For \namebu to terminate,
the budget $B$ has to accommodate at least minimal-cardinality paths between terminals,
which we assume holds true.}
%\reviseral{Note that \namebu is guaranteed to terminate provided that the budget allows to connect terminals through minimal-cardinality paths,
%which we assume holds true.}

\begin{figure}
	\centering
	\begin{minipage}[l]{.5\linewidth}
		\centering
		\includegraphics[width=.7\linewidth]{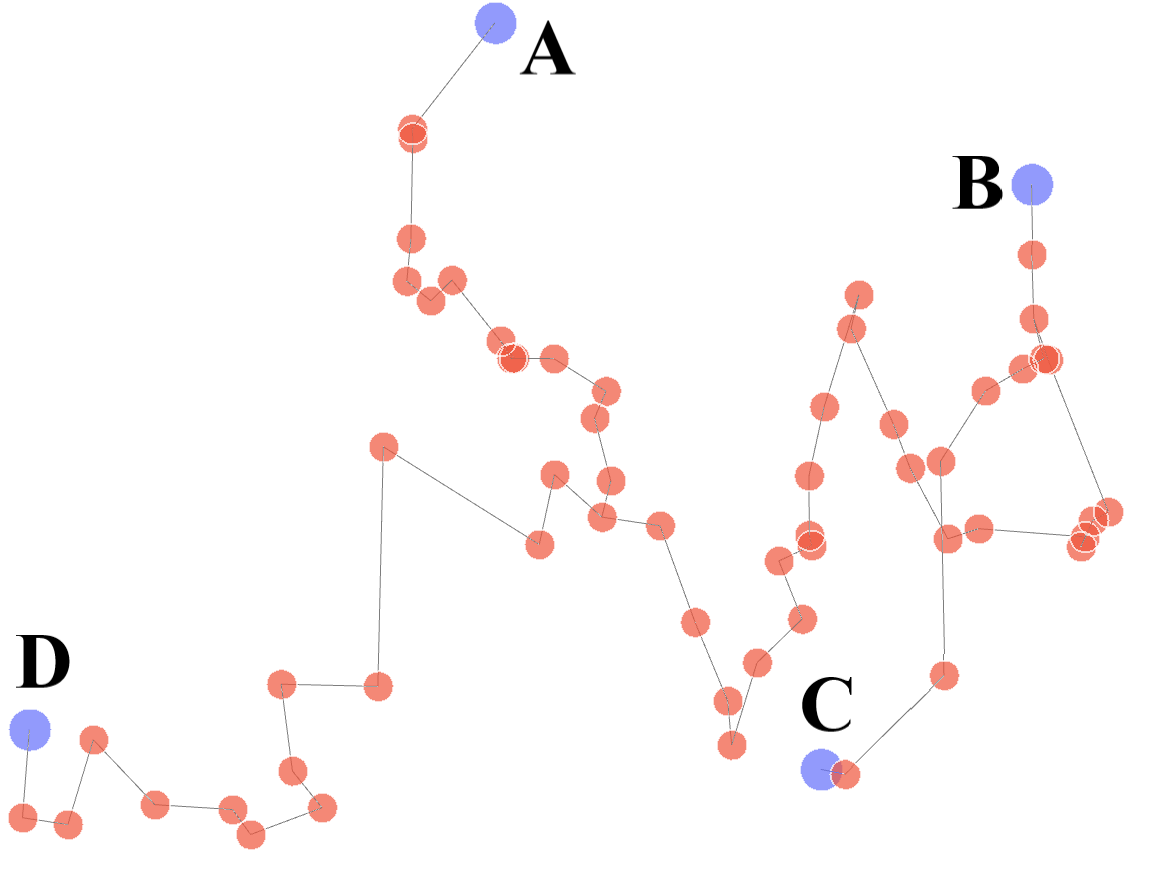}
	\end{minipage}%
	\hfil
	\begin{minipage}[r]{.5\linewidth}
		\centering
		\includegraphics[width=.7\linewidth]{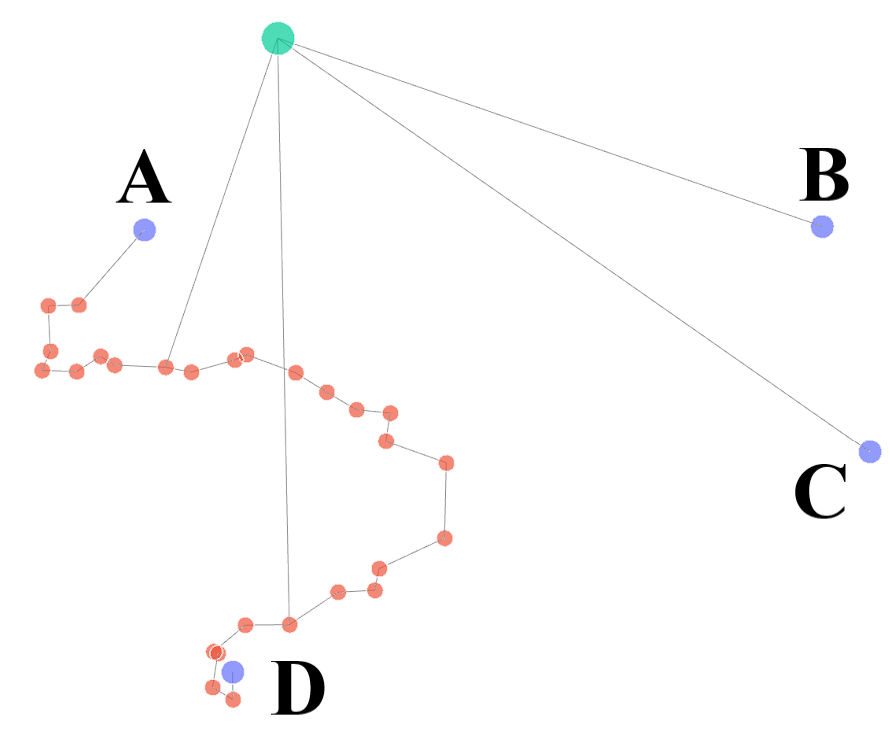}
	\end{minipage}
	\caption{Initial (left) and final \dsgshortname (right).
		Terminal nodes \revise{(\textbf{A}, \textbf{B}, \textbf{C}, and \textbf{D})} are in blue,
		place nodes in red,
		and the room node in green.
	}
	\label{fig:bottom-up-example}
\end{figure}

\begin{algorithm}
	\caption{\namebu}
	\label{alg:bottom-up}
	\DontPrintSemicolon
	\KwIn{\dsgshortname $ \dsg $, terminal pairs $ \pairs $, node budget $ \bw $.}
	\KwOut{Compressed \dsgshortname $ \dsgc $.}
	$ \dsgc \leftarrow $ \verb|build_spanner|$ \lr\dsg,\pairs\rr $; \tcp*{initialization}\label{alg1:initialization}
	\ForEach(\tcp*[f]{track node usage}){$ n\in\nodes[\dsgc] $}
	{
		$ \dict[n] \leftarrow \{(\source,\target) \in\pairs : n\in\path{\dsgc}{\source}{\target}\}$;\; \label{alg1:build-dict}
%		subset of terminals pairs whose shortest path in $ \dsgc $ passes through $ n $\tcp*{track node usage}\label{alg1:build-dict}
	}
	\For(\tcp*[f]{parse layers bottom-up}){$i = 0,\dots,L-1$\label{alg1:foreach-layer}}	%, 
	{
		\While(\tcp*[f]{parse layer till empty}){$\layer{\dsgc}{i} \neq \emptyset$\label{alg1:empty-layer}}
		{
%			\While{$\exists (s,t) \in \pairs:\path{\dsgc}{s}{t}\cap\layer{\dsgc}{i} \neq \emptyset$}
			\ForEach(\tcp*[f]{parse path from $ s $ to $ t $}){$ (s,t) \in \pairs $\label{alg1:foreach-path}} 
			{
				\If{$\exists n\in\layer{\dsg}{i+1}:\path{\dsgc}{s}{t}\cap\children[\dsgc](n)\neq\emptyset$\label{alg1:check-children}}
%				\If{$ \exists $ stretch of nodes $ \nodes[\dsg](n_{\ell+1})\subseteq\path{\dsgc}{s}{t} $\label{alg1:check-children}}
				{
					replace $ \path{\dsgc}{s}{t}\cap\children[\dsgc](n) $ with $ n $ in $ \path{\dsgc}{s}{t} $;\; \label{alg1:replace-nodes-in-path}
					$ \dict[n] \leftarrow \dict[n] \cup {(s,t)} $;\;
					\ForEach{$ m\in\path{\dsgc}{s}{t}\cap\children[\dsgc](n) $}
					{
						$ \dict[m] \leftarrow \dict[m] \setminus {(s,t)} $;\; \label{alg1:delete-node-start}
						\If(\tcp*[f]{\reviseral{prune \dsgshortname}}){$ \dict[m] =\emptyset $}
						{
							$ \nodes[\dsgc] \leftarrow \nodes[\dsgc]\setminus \lb m\rb $;\;
							$ \reviseral{\edges[\dsgc] \leftarrow \edges[\dsgc]\setminus \edges[\dsgc](m)} $;\; \label{alg1:delete-node-end}
						}
					}
					\If(\tcp*[f]{budget is met}){$|\nodes[\dsgc]| \le \bw $}
					{
						\Return $ \dsgc $.
					}
				}
			}
		}
	}
\end{algorithm}

%Note that it is technically possible that communication budget is not enough to even
%transmit connected paths for source-target pairs.
%However,
%we assume that practically some minimal resources are always available
%so that the algorithm terminates correctly
%(possibly returning very distorted paths).
			%!TEX root = ../main.tex

\myParagraph{Performance bound}
We now provide an analytical bound on the worst-case stretch that is incurred by every shortest path after running \namebu.
First,
we provide two definitions that are instrumental to the understanding of the bound. %define the ancestors of a place node.

\begin{definition}[Ancestor]
	The \emph{($i$th) ancestor} $\ancestor{n}{i}{\dsg}$ of node $n\in\layer{\dsg}{i_0}$ %in layer $\layer{\dsg}{i}$
	is the unique node in layer $\layer{\dsg}{i}$, $i>i_0$,
	that is connected to $n$ by a path composed of only cross-layer edges.
	%	\LC{The  seems to hint that each node has a *unique* ancestor}
	%	That is,
	%	\begin{equation}\label{eq:ancestor}
		%		a = \ancestor{n}{i}{\dsg} \iff \exists \path{\dsg}{a}{n} : (r,s)\in\path{\dsg}{a}{n} \iff r\in\layer{\dsg}{j},s\in\layer{\dsg}{j-1}, j \in \lb1,\dots,i\rb
		%	\end{equation}
	%	where $\path{\dsg}{a}{n}$ is a path between nodes $a$ and $n$ in $\dsg$.
	%\reviseral{The \emph{parent} $\parent{n}{\dsg}$ of node $n\in\layer{\dsg}{i_0}$ is $\parent{n}{\dsg}=\ancestor{n}{i_0+1}{\dsg}$.} % of $n$ in layer $\layer{\dsg}{i+1}$. %immediately above.
\end{definition}

In words,
the ancestors of node $n$ %are the nodes that encode coarser 
are coarse representations of $n$ in upper layers.
\reviseral{For example,
the first two ancestors of a place node are its room and building nodes,
respectively.} %with the place,
%and the second ancestor is its associated building. % associated with the room.
%\reviseral{Note that the parent is the closest ancestor.}

\begin{definition}[Diameter]
	For any node $n\in\nodes[\dsg]$,
	its \emph{diameter} $\diam{n}{\dsg}$ is
	the maximum cardinality of all shortest paths connecting two children nodes of $n$ in $\dsg$,
	that is,
	\begin{equation}\label{eq:diameter}
		\diam{n}{\dsg} \doteq \max \lb |\path{\dsg}{c_1}{c_2}| : \reviseral{c_1, c_2 \in \children[\dsg](n)} \rb, %n = \parent{c_1}{\dsg}, n = \parent{c_2}{\dsg} \rb,
	\end{equation}
	where $|\path{\dsg}{c_1}{c_2}|$ denotes the number of nodes in $\path{\dsg}{c_1}{c_2}$.
%	\LC{the notation for Path "P" seems undefined. Maybe we can use the distance notation we introduced in the problem statement?}
\end{definition}

In words,
the diameter of a node describes how ``large'' the node is when expanded into its children in the layer below.

We now assume the following bounds on quantities associated with the original \dsgshortname $\dsg$.
Recall that any edge $(m,n) \in \edges[\dsg]$ with $m,n\in\nodes[\dsg]$ is assigned a weight $\weight{\dsg}{m}{n}$.

\begin{ass}[\dsgshortname bounds]\label{ass:bounds}
	For any layer $i \in \lb1,\dots,L\rb$,
	\begin{align}
		\begin{split}
			\wmax{i} &\doteq \max \lb \weight{\dsg}{m}{n}:m,n\in\layer{\dsg}{i} \rb,\\
			\wcmax{i}&\doteq \max \lb \weight{\dsg}{m}{n}:m\in\layer{\dsg}{i-1},n\in\layer{\dsg}{i} \rb,\\
			\mmin{i} &\doteq \min \lb \left|\path{\dsg}{\ancestor{\source}{i}{\dsg}}{\ancestor{\target}{i}{\dsg}}\right| : (\source,\target)\in\pairs \rb,\\
			\bmin{i} &\doteq \min \lb \diam{n}{\dsg} : n \in \layer{\dsg}{i} \rb.
		\end{split}
	\end{align}
\end{ass}

%In words:

\begin{itemize}
	%	\item $\wmin{i} $ is the minimum weight of edges in layer $\layer{\dsg}{i}$;
	\item $\wmax{i}$ is the maximum weight of edges in layer $\layer{\dsg}{i}$;
	%	\item $\wcmin{i}$: minimum weight of cross-layer edges between layers $\layer{\dsg}{i-1}$ and $\layer{\dsg}{i}$
	\item $\wcmax{i}$ is the maximum weight of cross-layer edges between layers $\layer{\dsg}{i-1}$ and $\layer{\dsg}{i}$;
	%	\item $\mmin{i}$: minimum unweighed distance (number of edges) between (the ancestors of) any two terminals in layer $\layer{\dsg}{i}$
	\item $\mmin{i}$ is the minimum cardinality of a shortest path between the $i$th ancestors of every two connected terminals; % (that is,
%	their ancestors in the $i$th layer are far apart by no fewer than $\mmin{i}$ nodes).
	\item $\bmin{i}$ is the minimum diameter of nodes in layer $\layer{\dsg}{i}$.
\end{itemize}

Equipped with the definition above,
we can bound
the distortion on the compressed \dsgshortname $\dsgc$ provided by \namebu.

\begin{prop}[Worst-case \namebu stretch]\label{prop:bound-budlite}
%	For any terminal pair $(\source,\target)\in\pairs$,
	After $k$ total iterations of the innermost loop in~\cref{alg:bottom-up},
	the distance between any two terminals in the compressed graph $\dsgc$ is
	\begin{multline}\label{eq:bu-stretch-bound}
		%	2\sum_{i=\lz+1}^{\lmin}\wcmin{i} + \wmin{\lmin} + \lr\mmin{\lmin-1} - \bmax{\lmin}\rr\wmin{\lmin-1}
		%	\le
		\hspace{-3mm}\dist{\dsgc}{\source}{\target} 
		\;\le\; 
		2\sum_{i=1}^{\lmax}\wcmax{i} + \lr\mmin{\lmax-1}-\alpha_k\bmin{\lmax}\rr\wmax{\lmax-1} \\
		+\alpha_k\wmax{\lmax}, \quad \forall (s,t)\in\pairs,
	\end{multline}
	where
	\allowdisplaybreaks
	\begin{align}\label{eq:bounds-assumptions}
		\alpha_k&\doteq \ceil*{\dfrac{k}{|\pairs|}-\sum_{i=\lz}^{\lmax-1}\mmin{i}},\\
		%		\lmin\doteq\min\lb\ell:k\le|\pairs|\sum_{i=1}^{\ell}\mmin{i}\rb\\
		\lmax	&\doteq\max\lb\ell:k>|\pairs|\sum_{i=\lz}^{\ell-1}\mmin{i} \rb,\\
		\lz 	&\doteq \! \begin{multlined}[t]
			\max\lb\ell:\min_{(\source,\target)\in\pairs}\lr\dist{\dsg}{\source}{\target}+\beta\maxweight{\dsg}{\source}{\target}\rr\ge\right. \\
			\left. 2\sum_{i=1}^{\ell}\wcmin{i} + \mmin{\ell}\wmin{\ell}\rb.
		\end{multlined}
	\end{align}
\end{prop}

%\marginnote{\LC{need to explain in words what $\alpha$, $\lmax$ and $\lz$ represent and what the bound is intuitively saying.}}

\begin{proof}
	See \isExtended{the technical report~\cite[Appendix~D]{Chang22arxiv-dlite}}{\autoref{app:bound}}.
\end{proof}

\revise{In words,
$\lz$ is the index of the bottom layer in the initial spanner $\dsgc$ in~\cref{alg1:initialization} (excluding terminals);
$\lmax$ is the index of the top (coarsest) layer reached by \namebu after $k$ total iterations; %of the loop at~\cref{alg1:foreach-path};
%corresponding to the  layer retained in the final spanner;
$\alpha_k$ is the number of nodes in the latter layer that have been added to the compressed graph after $k$ iterations.
In the bound~\eqref{eq:bu-stretch-bound}, % represents the maximum stretch incurred by any shortest path:
the first term is the stretch due to cross-layer edges connecting the terminals to nodes in the upper layers,
while the other terms are the stretch due to the shortest path passing partially across the coarsest layer $\layer{\dsg}{\lmax}$ (third term)
and partially across layer $\layer{\dsg}{\lmax-1}$ (second term).}
			%!TEX root = ../main.tex

%\subsection{\titlecap{Top-Down Expansion Algorithm}}
\subsection{\titlecap{TOD-Lite: a Top-Down Expansion Algorithm}}
\label{sec:alg-top-down}

\begin{figure}
	\centering
	\includegraphics[width=.9\linewidth]{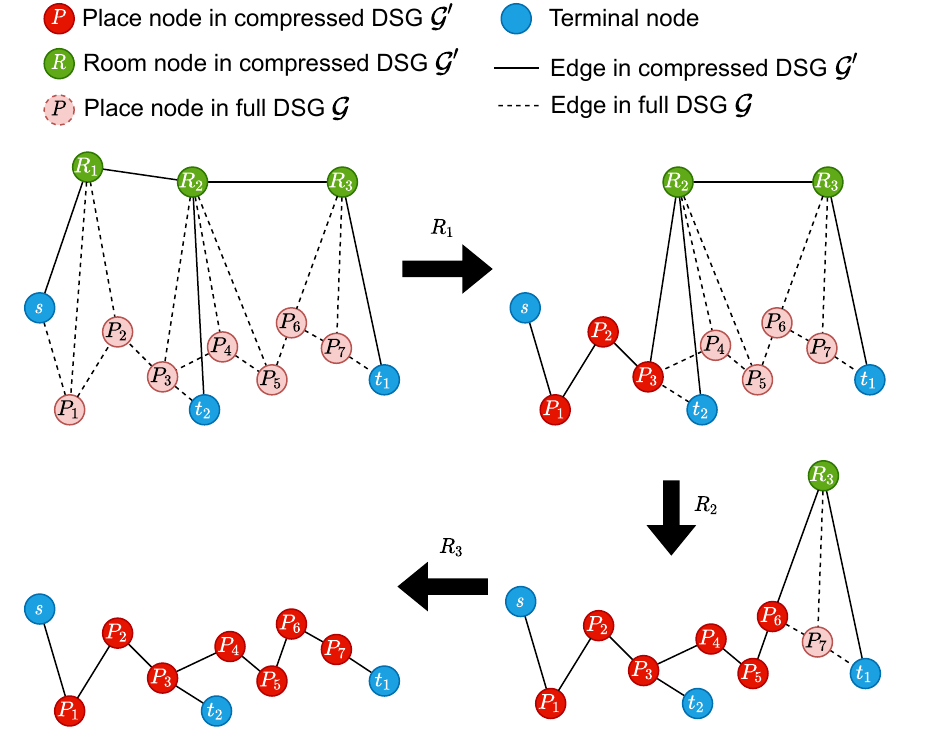}
	\caption{Illustration of the \nametd expansion procedure with one source $ \source $
		and two targets $ \target_1 $ and $ \target_2 $.
		At each iteration,
		a room node is expanded and replaced with its children place nodes.
		Adjacent nodes may be added to ensure connectivity
		(\eg $ P_3 $ at first iteration).
	}
	\label{fig:top-down}
\end{figure}

This section presents our second greedy algorithm.
Symmetrically to the bottom-up approach of~\cref{alg:bottom-up},
the idea behind \nametd (short for TOp-down D-Lite)
is to exploit the \dsgshortname hierarchy
by expanding node children to iteratively increase spatial granularity of the compressed graph (\autoref{fig:top-down}).
% We now develop this idea in detail.

\revise{The idea of \nametd is depicted with a toy example in~\autoref{fig:top-down},
	where room nodes $ R_1 $, 
$ R_2 $,
and $ R_3 $ are
progressively replaced with their respective children (place nodes).
The initial condition (top left)
contains the minimum set of nodes that guarantee connectivity between terminals,
%ensure connectivity between terminals under the minimum possible budget,
and it features %conveys little information for navigation through the 
coarse spatial abstractions through the retained nodes.}

\revise{We formally describe \nametd in~\cref{alg:top-down}.}
During initialization,
\cref{alg:top-down}
builds a spanner $ \dsgc_\text{target} $ of the \dsgshortname with\isExtended{ \texttt{build\_spanner}}{~\cref{alg:build-spanner}},
which is used as target for the final compressed graph $ \dsgc $
(\cref{alg2:build-target-spanner}). % which serves as target graph
%(\ie it is reached if and only if budget is sufficient),
Then,
it populates a ``hierarchical spanner'' $\mathcal{H}$ (\cref{alg2:spanner-hierarchy}):
this is simply a graph obtained from the original \dsgshortname $ \dsg $
by keeping the spanner $ \dsgc_\text{target} $
plus nodes and edges encountered starting from $ \dsgc_\text{target} $
and going up the \dsgshortname hierarchy
all the way to the top layer.
Elements unrelated to the ancestors of $ \dsgc_\text{target} $ are removed.
%with $ \dsgc_\text{target} $ at the bottom
%and such that each spanner above is obtained by compressing the bottom layer of the spanner below in the hierarchy.
For example,
if $ \dsgc_\text{target} $ is made of place nodes,
$\mathcal{H}$ includes $ \dsgc_\text{target} $,
the room nodes associated with those places (with cross-layer edges),
and possibly nodes above in the hierarchy,
\eg the buildings containing those rooms.
Graph $ \mathcal{H} $ is used to expand nodes from coarser to finer layers,
as explained next.
\isExtended{%
	The expansion priority is given by the number of paths that traverse a node (Lines~\ref{alg2:build-dictionary-1} to~\ref{alg2:build-dictionary-2}):
	\reviseral{this way, 
		expanding a node with high priority restores the spatial resolution of many paths (\ie those traversing that node).}
%	so that paths are stretched in balanced fashion.
	}{%
	To define an expansion priority for nodes in the same layer,
	a data structure $ \dict $ stores how many paths pass though each node in the graph,
	including both original paths in the target spanner (\cref{alg2:build-dictionary-1})
	and path abstractions in upper layers (\cref{alg2:build-dictionary-2}):
	for example,
	the priority of a room node $ R $ is given by the number of paths actually passing through $ R $ in $ \dsgc_\text{target} $
	and of paths traversing places nodes associated with $ R $.}

The main phase is an iterative top-down expansion through the hierarchical spanner $ \hierarchy $.
The output graph $ \dsgc $ is initialized with terminal nodes
and paths
\revise{that connect them under %minimum-cardinality paths,
	minimal communication budget:
	these are obtained by connecting each source-target pair to their common ancestor at minimal distance through cross-layer edges (Lines~\ref{alg2:cheapest-spanner-terminals} to~\ref{alg2:cheapest-spanner-cross-layer})}.
	%\footnote{
%	Again,
%	we assume that a minimal communication budget is always available to transmit at least such minimum-cost paths between terminals.}
Then,
starting from the top layer $ \layer{\dsgc}{L} $
and going down one layer at a time (\cref{alg2:expansion-layers}),
each node in $ \dsgc $ is \emph{expanded} %by the \verb|expand_node| subroutine
%if such operation does not exceed the budget
(\reviseral{\cref{alg2:expansion-nodes,alg2:expansion-nodes-add,alg2:expansion-nodes-remove}})
till such operation is infeasible (\cref{alg2:termination}).
{\revise{In particular,
	%\verb|expand_node| (\cref{alg:expand-node})
	if a node $ n \in\dsgc$ has a set of children $ \children[\hierarchy](n) $ in the hierarchical spanner $ \hierarchy $,
	then~\cref{alg2:expansion-nodes-remove} removes \reviseral{$ n $ from $\dsgc$}
	and~\cref{alg2:expansion-nodes-add} adds to $ \dsgc $ nodes in $ \children[\hierarchy](n) $ and their incident edges \reviseral{$\edges[\hierarchy](\children[\hierarchy](n))$}.}
%	(\emph{expands} $ n_{\ell} $).
	%and incident edges $ \edges[\hierarchy](\nodes[\hierarchy](n_\ell)) $
	%in the hierarchical spanner $ \hierarchy $.}
%For example,
%if $ n_\ell $ is a room node,
%the place nodes in that room %belonging to the full spanner $ \dsgc_\text{target} $ 
%are added to the graph
%%together with their incident edges,
%in place of $ n_\ell $. % is removed.

Expanding nodes gradually restores the geometric granularity of the \dsgshortname spanner,
because a spatially coarse representation (\eg room node) is replaced by 
a group of nodes with fine resolution (\eg place nodes).
This expansion comes at the price of heavier communication burden.
Nonetheless,
using the hierarchical spanner allows us to narrow the expansion procedure 
to a small set of navigation-relevant nodes,
both saving runtime and helping meet communication constraints.

\begin{algorithm}
	\caption{\nametd}
	\label{alg:top-down}
	\DontPrintSemicolon
	\KwIn{\dsgshortname $ \dsg $, terminal pairs $ \pairs $, node budget $ \bw $.}
	\KwOut{Compressed \dsgshortname $ \dsgc $.}
	$ \dsgc_\text{target} \leftarrow $ \verb|build_spanner|$ \lr\dsg,\pairs\rr $;\; \label{alg2:build-target-spanner}
	$ \hierarchy \leftarrow $ hierarchical spanner from $ \dsgc_\text{target} $;\;\label{alg2:spanner-hierarchy}
	\ForEach(\tcp*[f]{for expansion priority}){$ n\in\nodes[\dsgc_\text{target}] $\label{alg2:build-dictionary-1}}
	{	
		$ \dict[n] \leftarrow \big|\{(\source,\target) \in\pairs : n\in\path{\dsgc_\text{target}}{\source}{\target}\}\big|$;\;
%		$ \dict[n] \leftarrow $ number of terminals pairs whose shortest path in $ \dsgc_\text{target} $ passes through $ n $;
	}
	\For{$i = 1,\dots,L$}
	{
		\ForEach{$ n\in\layer{\hierarchy}{i} $} %\Comment{paths passing through abstracted nodes}
		{
			$ \dict[n] \leftarrow \sum_{n'\in\nodes[\hierarchy]\lr n\rr } \dict[n'] $;\;
			\label{alg2:build-dictionary-2}
%			$ \dict[n] \leftarrow \dict[n] \cup\bigcup_{n'\in\nodes[\dsg]\lr n\rr } \dict[n'] $;\;
		}
	}
	$ \nodes[\dsgc] \leftarrow \pairs $; \tcp*{add terminals}\label{alg2:cheapest-spanner-terminals}
	\ForEach(\tcp*[f]{add cheapest path from $ s $ to $ t $}){$ (s,t)\in\pairs $}
	{
		$ a \leftarrow $ lowest common ancestor of $ \source $ and $ \target $ in $ \hierarchy $;\;
		$ \nodes[\dsgc] \leftarrow \nodes[\dsgc] \cup \lb a\rb $;\;
		\reviseral{$\edges[\dsgc]\leftarrow$ edges connecting $ s $ and $ t $ with $ a $ in $ \dsgc $};\;\label{alg2:cheapest-spanner-cross-layer}
	}
	\For(\tcp*[f]{parse layers top-down}){$i = L,\dots,1$\label{alg2:expansion-layers}}
	{
		\ForEach(\tcp*[f]{sorted by $ \dict[n] $}){$ n \in\layer{\dsgc}{i}$}
		{
			\If{can expand $ n $ without exceeding $ B $\label{alg2:expansion-nodes}}
			{
				$ \nodes[\dsgc] \leftarrow  \nodes[\dsgc] \setminus \{n\} $; $ \edges[\dsgc] \leftarrow  \edges[\dsgc] \setminus \edges[\dsgc](n) $;\;\label{alg2:expansion-nodes-remove}
				$ \nodes[\dsgc] \leftarrow \nodes[\dsgc] \cup \children[\hierarchy](n) $; $\edges[\dsgc]\leftarrow\edges[\dsgc]\cup\edges[\hierarchy](\children[\hierarchy](n))$;\;\label{alg2:expansion-nodes-add}
			}
		}
		\If{no node in $\layer{\dsgc}{i}$ has been expanded\label{alg2:termination}}
		{
			\Return $ \dsgc $.\;
		}
	}

\end{algorithm}

%\LC{this part can be made more to the point: explain it in simple terms and omit unnecessary details - now
	%the reader does not really understand what's going on}

\iffalse
\LC{Algo~\ref{alg:expand-node} can be included in Algo 3 instead.}

\begin{algorithm}
	\caption{Expand node}
	\label{alg:expand-node}
	\begin{algorithmic}[1]
		\Require Node $ n $, spanner $ \dsgc $, hierarchical spanner $ \hierarchy $.
		\State $ \dsgc \leftarrow \dsgc \cup \nodes[\hierarchy](n_\ell) \cup\edges[\hierarchy](\nodes[\hierarchy](n_\ell))$;
		\State $ \dsgc \leftarrow  \dsgc \setminus \{n\} $.
	\end{algorithmic}
\end{algorithm}
\fi

Note that,
with enough communication resources,
\nametd would exactly output the target spanner $ \dsgc_\text{target} $.
Under limited budget,
some nodes in $ \dsgc_\text{target} $ cannot be expanded,
\eg a room may be used as a coarse representation of its places.
\isExtended{}{Experimental results of \nametd are provided in~\autoref{app:experiments}.}

\isExtended{}{
	\myParagraph{Comparison with~\cite{Larsson20tro-qTree}}
	The iterative expansion of nodes from coarser to finer layers used by \nametd
	resembles the approach used in~\cite{Larsson20tro-qTree}.
	However,
	there are fundamental differences between these two methodologies.
	First,
	we expand nodes along a preexisting semantic hierarchical structure (the \dsgname),
	while the hierarchy in~\cite{Larsson20tro-qTree} simply emerges from the regular geometry
	of the environment (such as a grid map in~\cite{Larsson20tro-qTree,Larsson21ral-infoTheoreticAbstractionsPlanning}),
	without awareness of semantics or physical quantities such as navigation time to move through coarse- and fine-scale maps.
	Second,
	our expansion leverages a target spanner computed up front and is guided by the navigation task,
	in particular by the stretch incurred by the shortest paths,
	while nodes in~\cite{Larsson20tro-qTree} are expanded based on an information-theoretic cost
	to be defined by suitable probability distributions
	whose support and density function change with expansions but are initially defined on the full graph to be compressed.
	%	However,
	%	the expansion mechanism leveraging those probability distributions does not exploit 
	%	because of different hierarchical structure and task which require spatial information for navigation.
	More details about the algorithm in~\cite{Larsson20tro-qTree} are given in~\autoref{sec:experiments}.
}
			%!TEX root = ../main.tex

\subsection{Discussion: \namebu \vs \nametd}\label{sec:bud-vs-tod}
\namebu compresses the \dsgshortname in a more granular fashion compared to \nametd:
that is,
it adds distortion to paths more slowly,
because it compresses a limited portion of one path at a time.
On the other hand,
the expansion strategy of \nametd restores all children of a parent node at once.
This difference makes \namebu generally slower but able to reach a final graph size closer to the budget,
whereas \nametd is typically faster but retains fewer nodes and leads to more distorted paths.

Those differences make the two strategies suited to different scenarios.
For instance,
a map that includes both large and small rooms may cause \nametd to get stuck after expanding the nodes with the largest number of children,
while the path-wise compression of \namebu is less sensitive to heterogeneous maps.
On the other hand,
to compress a large but homogeneous map with many relevant locations,
one may use \nametd to favor compression speed against a slightly worse result.
	%!TEX root = ../main.tex

\section{Experiments}
\label{sec:experiments}

This section shows that 
%the compressed \dsgshortname using 
our method retains information for efficient navigation
while meeting the communication budget constraint. 
We also show that the algorithms run in real time. % for online operation.

\subsection{Experimental Setup}

Besides benchmarking \name against the solution to~\eqref{eq:optimization-problem}
(label: ``Optimum'') found via integer linear programming (ILP),
we also adapt \reviseral{and compare the compression strategy introduced in~\citep{Larsson20tro-qTree}}
%as a baseline for comparison 
(label: ``IB''), 
as discussed below. 

\myParagraph{Q-Tree search adaptation} %\myParagraph{Adapting Q-Tree Search}
The compression approach in~\citep{Larsson20tro-qTree} builds on the Information Bottleneck (IB)~\citep{Tishby01accc-IB}.
This approach aims to find a compact representation $T$ of a random variable $X$ by solving a relaxed version of the IB problem,
\begin{equation}
	\min_{p(T|X)} I(T; X) - \beta I(T; Y),
\end{equation}
where $I(T;X)$ is the mutual information between $T$ and $X$ 
and $I(T;Y)$ %is the mutual information between the compressed representation $T$
%and an additional distribution $Y$ encoding relevant information about $ X $.
represents the information that $ T $ retains about a third variable $ Y $ that encodes relevant information about $ X $.
Parameter $\beta$ can be seen as a knob to trade amount of relevant information retained in $T$ for compression rate.

To adapt this approach to navigation-oriented \dsgshortname compression
\reviseral{(since the Q-tree does not encode connectivity within a layer of the scene graph)},
we define a uniform distribution $ p(x) $ over the place nodes.
Next, 
we associate $Y$ with shortest paths between terminals:
if place $x_i$ is on the shortest path $y_j$,
then $p(y_j|x_i) = 1$.
%and $p(y_j|x_i) = 0$ otherwise.
From the place layer,
we build $ X $ by propagating $p(x)$ and $p(y|x)$ to upper layers by a weighted sum (cf.~\citep{Larsson20tro-qTree}).
We manually add the terminals if
they are not automatically added,
and in view of~\eqref{eq:optimization-problem}
we use the number of nodes
as a stopping condition besides the one in~\citep{Larsson20tro-qTree}.

\myParagraph{Simulator}
We showcase the online operation of \name in the %uHumans2 simulator.
Office environment of the uHumans2 simulator (\cref{fig:cover})~\citep{Rosinol21ijrr-Kimera},
with 4 scenarios %ranging from
%short, medium, to long in terms of 
featuring different distances between navigation goal and starting position of the robot.

The queried robot $ \robotRespond $ sending the compressed \dsgshortname 
\isExtended{is given potential locations the querying robot $ \robotQuery $ may come from.
}{%
	~has no exact knowledge of the location of the querying robot $ \robotQuery $,
	and is only given some potential locations. }%
The places closest to these source locations along with the place closest to the navigation goal are the terminals of \name.
In the short and medium sequences, $ \robotQuery $ gets two putative source locations, hence three total terminals.
In the two long sequences, $ \robotQuery $ gets three putative source locations, hence four total terminals.
For all sequences, \reviseral{we choose $ 60 $ nodes as communication budget,
	which is $ 1.6\% $ of the original \dsgshortname.} %\hspace{-5mm}

Upon receiving the compressed \dsgshortname,
robot $ \robotQuery $ finds the place node $ \source $
%in the \dsgshortname 
closest to its location and
computes the shortest path on the compressed \dsgshortname from $ \source $ to the place node $ \target $ that represents the goal.
Robot $ \robotQuery $ treats the nodes along the shortest path as navigation waypoints.
We combine waypoint following with the ROS navigation stack for local obstacle avoidance:
the latter is needed 
\isExtended{\revise{to allow navigation in the areas where low-level geometric information (\ie places nodes) has been abstracted away during the \dsgshortname compression}.}{%
	where free-space locations are not available to $ \robotQuery $ 
	(\ie for place nodes that are not communicated for a portion of the map).}

%!TEX root = ../main.tex

\begin{table}[t!]
	\footnotesize
	\setlength{\tabcolsep}{1pt}
	\centering
	\caption{Summary of results.
		Arrows indicate that lower is better.
		The table reports mean and standard deviation across three runs.
	}
	\label{tab:ablation}
	\begin{tabular}{cl |c|cccc}
		\toprule
		& & Full & Optimum & IB & \namebu & \nametd \\
		\midrule
		\multirow{4}{*}{short}
		& Comp[\si{\second}] $\downarrow$ & 0 $\pm$ 0& 247 $\pm$ 0& \textbf{1 $\pm$ 0}& \textbf{3 $\pm$ 0}& \textbf{3 $\pm$ 0}\\
		& Nom[\si{\second}] $\downarrow$ & {11 $\pm$ 0}& \textbf{11 $\pm$ 0}& 43 $\pm$ 0& \textbf{11 $\pm$ 0}& \textbf{11 $\pm$ 0}\\
		& Mis[\si{\second}] $\downarrow$ & 64 $\pm$ 8& 56 $\pm$ 5& 115 $\pm$ 16& \textbf{62 $\pm$ 2}& \textbf{59 $\pm$ 3}\\
		& Size(\#) ($\le\!60$) & 3814 $\pm$ 0& 51 $\pm$ 0& 60 $\pm$ 0& 49 $\pm$ 0& 49 $\pm$ 0\\
		\midrule
		\multirow{4}{*}{medium}
		& Comp[\si{\second}] $\downarrow$ & 0 $\pm$ 0& 294 $\pm$ 1& \textbf{1 $\pm$ 0}& \textbf{3 $\pm$ 0}& \textbf{3 $\pm$ 0}\\
		& Nom[\si{\second}] $\downarrow$ & {18 $\pm$ 0} & \textbf{18 $\pm$ 0} & 42 $\pm$ 0& \textbf{18 $\pm$ 0} & 29 $\pm$ 0\\
		& Mis[\si{\second}] $\downarrow$ & 87 $\pm$ 7& \textbf{77 $\pm$ 2}& 92 $\pm$ 22& \textbf{85 $\pm$ 8}& 144 $\pm$ 20\\
		& Size(\#) ($\le\!60$) & 3814 $\pm$ 0& 56 $\pm$ 0& 60 $\pm$ 0& 48 $\pm$ 0& 58 $\pm$ 0\\
		\midrule
		\multirow{4}{*}{long1}
		& Comp[\si{\second}] $\downarrow$ & 0 $\pm$ 0& - & \textbf{2 $\pm$ 0}& \textbf{3 $\pm$ 0}& \textbf{3 $\pm$ 0}\\
		& Nom[\si{\second}] $\downarrow$ & {27 $\pm$ 0}& - & $\infty$   & \textbf{31 $\pm$ 0}& 39 $\pm$ 0\\
		& Mis[\si{\second}] $\downarrow$ & {134 $\pm$ 5} & - & $\infty$   & \textbf{167 $\pm$ 18} & 273 $\pm$ 22\\
		& Size(\#) ($\le\!60$) & 3814 $\pm$ 0& - & 60 $\pm$ 0& 58 $\pm$ 0& 20 $\pm$ 0\\
		\midrule
		\multirow{4}{*}{long2}
		& Comp[\si{\second}] $\downarrow$ & 0 $\pm$ 0& - & \textbf{2 $\pm$ 0}& \textbf{3 $\pm$ 0}& \textbf{3 $\pm$ 0}\\
		& Nom[\si{\second}] $\downarrow$ & {32 $\pm$ 0} & - & 36 $\pm$ 0& \textbf{33 $\pm$ 0} & 34 $\pm$ 0\\
		& Mis[\si{\second}] $\downarrow$ & {150 $\pm$ 6} & - & 218 $\pm$ 20& \textbf{164 $\pm$ 30} & 291 $\pm$ 39\\
		& Size(\#) ($\le\!60$) & 3814 $\pm$ 0& - & 60 $\pm$ 0& 60 $\pm$ 0& 9 $\pm$ 0\\
		\midrule
	\end{tabular}
\end{table}

\begin{figure*}
	\centering
	\includegraphics[width=0.9\linewidth, trim={200 20 250 50},clip]{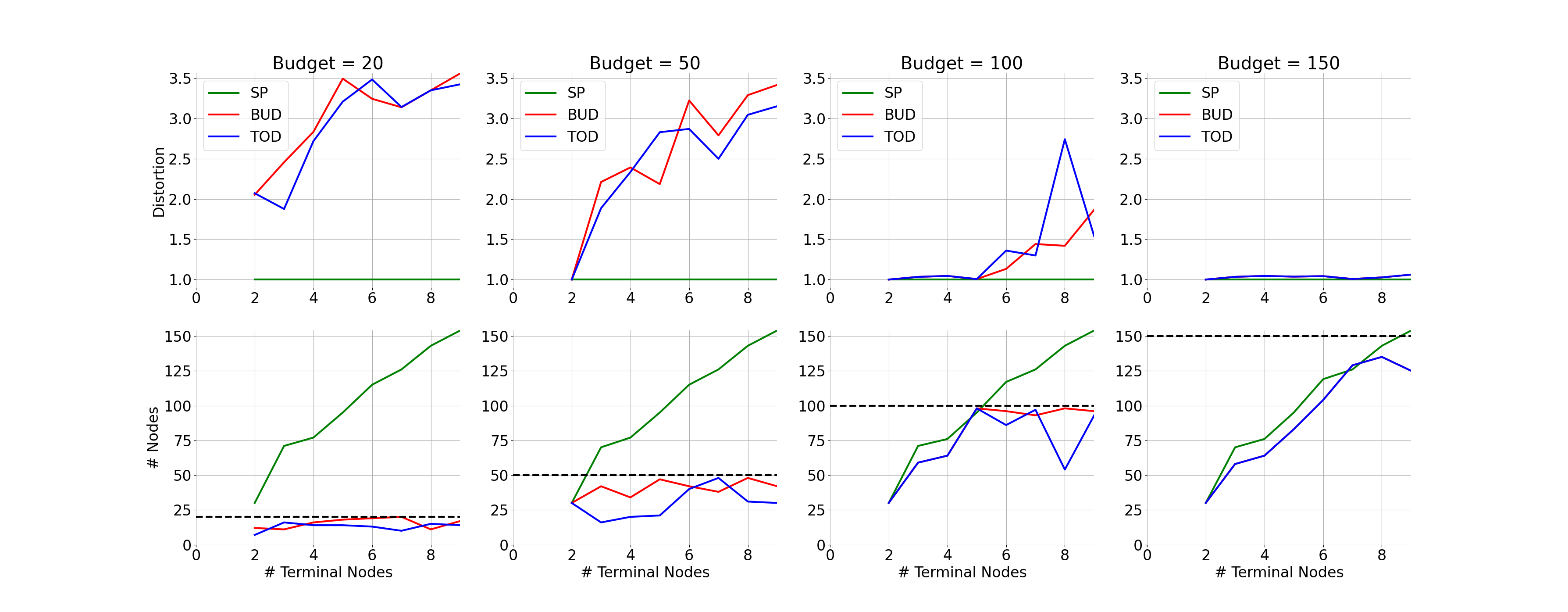}
	\caption{\revise{Comparison on distortion (top row) and number of nodes after compression (bottom row) for \namebu and \nametd
			against computing the shortest paths (SP) and pruning all nodes that are not on them.
			The dotted lines mark the communication budget.}
	}
	\label{fig:shortest_path_compare}
\end{figure*}

\subsection{Results and Discussion}

\myParagraph{Comparison with baselines}
The results on the four scenarios are documented in~\autoref{tab:ablation}.
We show the compression time (label: ``Comp''),
the nominal (label: ``Nom'', computed from the compressed \dsgshortname)
and simulated (label: ``Mis'', computed as the actual time $ \robotQuery $ takes to reach its destination in the simulator) mission times,
and the size of the compressed \dsgshortname\xspace \reviseral{(upper bounded by the budget $B$)},
all averaged across three runs.\footnote{
	The nominal mission time is computed by projecting the waypoints found by $\robotQuery$ in the compressed \dsgshortname
	onto the full \dsgshortname,
	calculating the total path length of traversing through those on the full \dsgshortname,
	and dividing by the maximum velocity of the agent. % in the simulator.
	% In other words, it is the theoretical navigation time on the graph used for path planning.
	% We do not directly use the compressed \dsgshortname to estimate the nominal time
	% because cross-layer edge weights are difficult to calculate from the full \dsgshortname,
	% however,
	% we argue that the queried robot $ \robotRespond $ can estimate those weights while building the \dsgshortname online,
	% see~\cite[Appendix~B]{Chang22arxiv-dlite}.
	In other words, it is the theoretical navigation time on the original \dsgshortname
	%	as initially received by the compression algorithm.
	and measures quality of compression.
	Note that we do not directly use the compressed \dsgshortname to estimate the nominal time
	because the cross-layer edge weights would be different and likely smaller in value compared to those calculated on the full \dsgshortname,
	see~\isExtended{\cite[Appendix~B]{Chang22arxiv-dlite}}{\autoref{app:edge-weight}}.
}
%with standard deviations.
The two best results for each row are in bold.

The combinatorial nature of problem~\eqref{eq:optimization-problem} makes the ILP solver impractical in robotic applications:
for the long runs,
the calculation of Optimum %solution of~\eqref{eq:optimization-problem} 
did not finish
%compress the \dsgshortname
within an hour.

In all scenarios,
robot $ \robotQuery $ reaches the goal using the compressed \dsgshortname output by \name.
The simulated mission time is at times faster on the compressed graph compared to the full \dsgshortname
because the former has fewer waypoints:
a sparser list of waypoints in a less cluttered space can actually yield faster navigation.
The different performance of
\namebu %(bottom-up compression) 
and \nametd %(top-down expansion)
is due to the different abstraction mechanisms,
whereby the path-wise node compression in the former yields finer granularity and usually better performance.
Discrepancies between nominal and simulated mission times are due to local navigation,
whose exploration time is difficult to estimate \textit{a posteriori} from the full \dsgshortname. %\LC{do you mean a priori?}
%but may be reliably estimated by the robot while building the \dsgshortname online.
%Notably,
\name always outperforms IB in terms of both nominal and simulated mission times.
Specifically,
the navigation planned on the compressed \dsgshortname produced by \namebu %the bottom-up compression~\cref{alg:bottom-up}
is only a minute longer than the optimal path on the original \dsgshortname.
\reviseral{IB is also unable to find a compressed \dsgshortname that preserves the necessary connectivity for the long1 case.}
%and \name-BU is the best also as regards simulated mission time.
%in all sequences.

\revise{
	\myParagraph{Ablation study}
	% \subsection{Ablation Study}
	We compare distortion and number of nodes %the distortion 
	on the shortest paths between terminal nodes, for increasing number of 
	terminal nodes and increasing budget constraints in~\autoref{fig:shortest_path_compare}.
	The shortest paths are optimal in terms of navigation performance (no distortion, top row),
	but easily violate the communication constraints (exceeding the budget, bottom row).
	On the other hand,
	\namebu and \nametd trade-off the path lengths between the terminal nodes to meet the budget constraint,
	and as we relax the latter,
	the distortion decreases. %\marginnote{\LC{not sure ``optimal'' is the right term here}}
	For the case with a budget of 150 nodes (last column), \namebu and \nametd obtain the same results,
	since the initial spanner already satisfies the budget constraint.}

\isExtended{\reviseral{A study on the runtime of \name and a comparison of \namebu against the bound~\eqref{eq:bu-stretch-bound}
		are provided in~\cite[Section~V-B]{Chang22arxiv-dlite}.}}{
	Figure~\ref{fig:compress_timing} shows an ablation on the compression time  along with the number of nodes removed or added from the initial spanner for different budget sizes and with an increasing number of terminal nodes.
	We observe that the compression time is still reasonable in practice (in the order of seconds) even with up to 50 terminal nodes and compressing the \dsgshortname to less than 5\% of its original size.
	The runtime of the compression algorithms is dominated by the initial spanner construction in \texttt{build\_spanner}.
	
	Figure~\ref{fig:bounds} plots the maximal distortion against the theoretical bounds proposed in~\cref{prop:bound-budlite}.
	While the tightness of the bounds varies, it still gives a useful estimate of the worst-case distortion.
	In particular,
	we note that the bound is typically tighter for larger numbers of terminal nodes and lower budget.
}
\isExtended{}{
	\begin{figure*}
		\centering
		\includegraphics[width=0.9\linewidth, trim={0 0 0 0},clip]{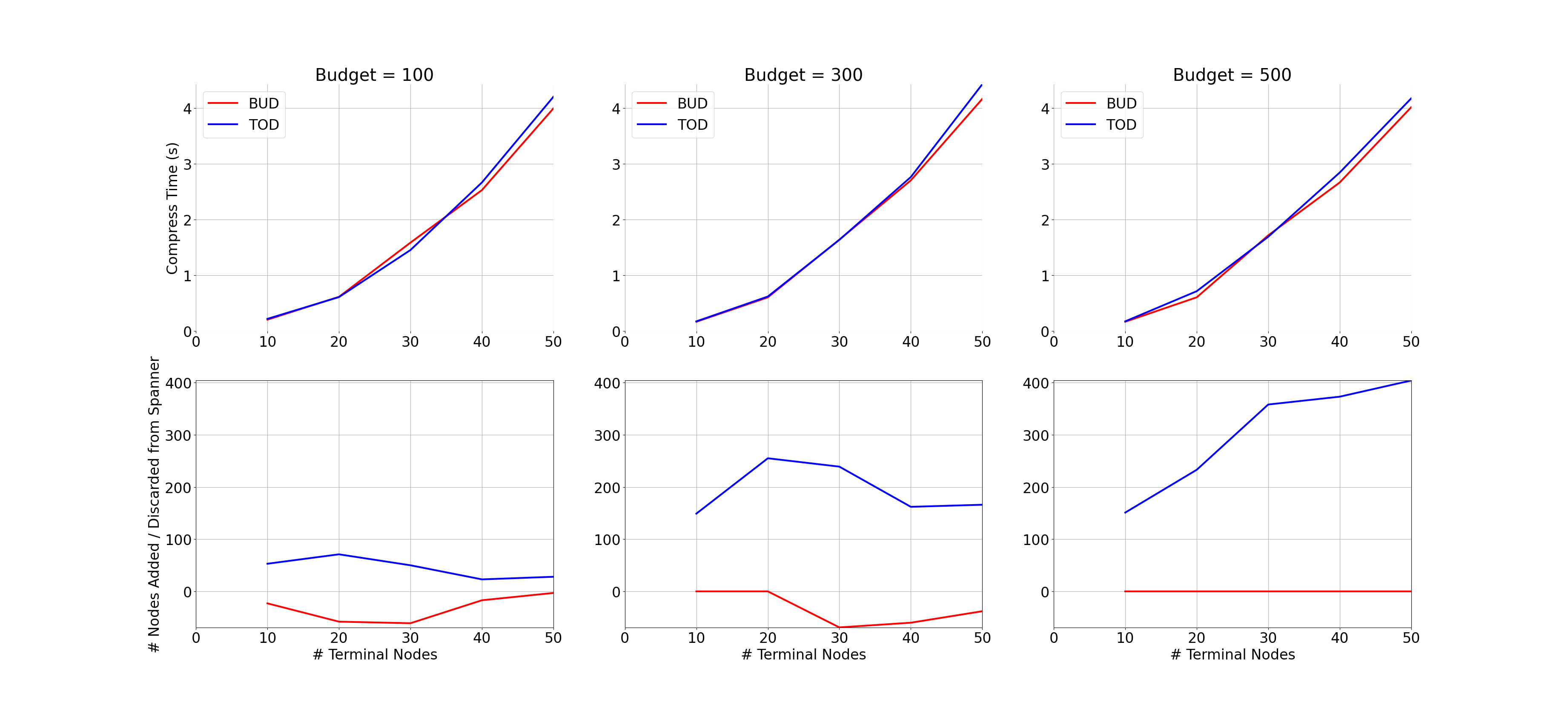}
		\caption{\small First row shows the compression time for a DSG with 3814 nodes for different number of terminals and budgets; second row shows the corresponding number of nodes removed (when abstracting nodes for BUD) or added (when expanding nodes for TOD). The overall compression time is dominated by the construction of the spanner.}
		\label{fig:compress_timing}
	\end{figure*}

	\begin{figure*}
		\centering
		\includegraphics[width=0.9\linewidth, trim={0 0 0 0},clip]{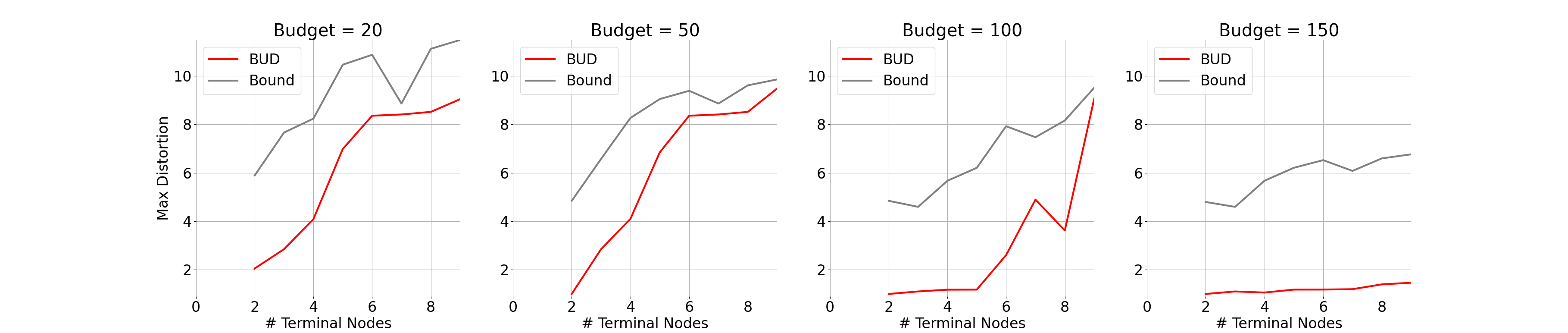}
		\caption{\small We show the distortion implied by the bound in Proposition 1 compared with the maximum distortion from the result of BUD-Lite for varying budgets and number of terminals using the DSG of the Office environment with 3814 nodes. }
		\label{fig:bounds}
	\end{figure*}
}
	%!TEX root = ../main.tex

% \vspace{-2mm}
\section{Conclusions}
\label{sec:conclusions}

\isExtended{
	Motivated by collaborative multi-robot exploration,
	we proposed a method to compress \dsgnames under communication constraints.
	Our algorithms can accommodate a sharp node budget
	while retaining navigation performance.
	Realistic simulations validate the effectiveness of our approach.
}{
	Motivated by the goal of enabling efficient information sharing for robots that collaboratively explore an unknown environment,
	we have proposed algorithms to suitably compress \dsgnames built and transmitted by robots during exploration,
	for the case when resource constraints of a shared communication channel make lossless transmission infeasible.
	Our algorithms can accommodate the presence of a sharp budget on the size of the transmitted map,
	run in real time,
	and perform graph compression with attention to the performance on specified navigation tasks.
	Simulated experiments carried out with a realistic simulator
	show that our approach is able to meet communication constraints
	while providing satisfactory performance of navigation tasks
	planned on the compressed~\dsgshortname.
	
	The proposed approach opens several interesting avenues for future work.
	In fact,
	\dsgnames are recently developed tools,
	and their use in multi-robot cooperation and collaboration is still relatively unexplored.
	%yet to be carefully explored.
	%needs to be more extensively investigated.
	For example,
	it is interesting to adapt our compression algorithms to  data collection in dynamic environments ---as the ones considered in~\cite{Rosinol21ijrr-Kimera}--- that induce time-varying graphs.
	Also,
	extension of compression techniques to various and more general tasks should be addressed.
	Finally,
	validation of the proposed algorithms on real robots is desired to test their impact in the real world.
}
	
	% Bibliography
	{\footnotesize
		\isExtended{
			\input{main_shorts.bbl}
		}{

%			\bibliographystyle{IEEEtran}
%			\bibliography{../../references/refs,../../references/myRefs}
		}	
	}
	
	\isExtended{}{
		\newpage
		\onecolumn
		\appendices
		%!TEX root = ../main.tex

\section{Exact Budget-Constrained Spanner}
\label{app:ilp}

Problem~\eqref{eq:optimization-problem} can be solved exactly by the following ILP
(adapted from the exact spanner formulation in~\cite[Section 4]{Ahmed21sea-multiLevelSpanners}),
\begin{mini!}
	{\substack{x_i \,\forall i\in\nodes[\dsg]\\
		x_{i,j}^{st}\,\forall (i,j)\in\edges[\dsg],\forall(s,t)\in\pairs}}{\distortion}
	{\label{eq:ILP}}
	{\label{eq:ILP-objective-cost}}
	\addConstraint{\sum_{(i,j) \in \overline{\mathcal{\edges}}_\dsg} x_{(i,j)}^{u v} \weight{\dsg}{i}{j}}
		{\le \dist{\dsg}{s}{t} + \distortion \maxweight{\dsg}{s}{t}}{\quad \forall (s,t)\in\pairs, \forall (i,j)\in\edges[\dsg]}\label{eq:ILP-constr-distortion}
	\addConstraint{\sum_{(i,j) \in \Out{i}} x_{(i,j)}^{st}-\sum_{(j,i) \in\In{i}} x_{(j,i)}^{st}}
		{= \begin{cases}
				1 & i=s \\ 
				-1 & i=t \\ 
				0 & \text{else}
			\end{cases}}{\quad \forall (s,t)\in\pairs, \forall i \in \nodes[\dsg]}\label{eq:ILP-constr-path-1}
	\addConstraint{\sum_{(i,j) \in \Out{i}} x_{(i,j)}^{st}}{ \leq 1}{\quad \forall (s,t)\in\pairs, \forall i \in \nodes[\dsg]}\label{eq:ILP-constr-path-2}
	\addConstraint{x_i}{\ge x_{(i,j)}^{st} + x_{(j,i)}^{st}}{\quad\forall (s,t)\in\pairs, \forall i \in\nodes[\dsg],\forall (i,j)\in\edges[\dsg]}\label{eq:ILP-constr-pick-edges}
	\addConstraint{\sum_{i\in\nodes[\dsg]} x_i}{\le\bw}{}\label{eq:ILP-constr-bw}
	\addConstraint{x_i,x_{(i,j)}^{st}}{\in\{0,1\}}{\quad\forall(s,t)\in\pairs, \forall i \in\nodes[\dsg], \forall (i,j)\in\edges[\dsg],}\label{eq:ILP-constr-integer}
\end{mini!}
where
$ x_i $ is associated with each node $ i\in\nodes[\dsg] $ and is $ 1 $ if it is included in the spanner,
$ x_{(i,j)}^{st} $ is an edge variable equal to $ 1 $ if and only if edge $ (i,j) $ is taken as part of the path between $ s $ and $ t $,
$ \overline{\mathcal{\edges}}_\dsg $ is the augmented set of bidirected edges,
obtained by adding edge $ (j,i) $ for each edge $ (i,j) \in\edges[\dsg]$,
\eqref{eq:ILP-constr-distortion} forces maximum distortion for all paths between terminal nodes,
\eqref{eq:ILP-constr-path-1}--\eqref{eq:ILP-constr-path-2} ensure that the chosen edges form a path for each pair of terminals $ (\source,\target) $,
\eqref{eq:ILP-constr-pick-edges} ensures that a node $ i $ is taken if any edges incident to it are taken,
and~\eqref{eq:ILP-constr-bw} encodes the limited budget on the number of selected nodes.
		%!TEX root = ../main.tex

\section{\titlecap{calculation of edge weights}}
\label{app:edge-weight}
The edge weights associated with the intra-layer edges of the \dsgname $ \dsg $ are simply the Euclidean distance between the two nodes the edge connects. 
For example, for the layer consisting of places, the weight associated to an edge would be the Euclidean distance between the two connected places; 
for the layer consisting of rooms, the edge weight would be the Euclidean distance between the centroid of two rooms.
The calculation of inter-layer edge weights is more nuanced: they cannot be simply Euclidean distances,
since that would fail to capture the actual effort to traverse,
for example,
from a room centroid to a place in the room 
without precise knowledge of free-space locations in the place layer.
Intuitively,
the lack of such precise spatial information requires the robot to parse the room by locally exploring it,
until the target location (place) is reached.
In general,
if the robot is aware only of an abstract,
spatially coarse representation of granular geometric information about free-space locations
(such as a room or building node that represents a set of places nodes),
time-consuming exploration is needed to supply the missing spatial information.
This fact well summarizes the trade-off that a robot faces when compressing its local map:
high compression rate enables quicker transmission,
but inevitably causes the robot that receives the compressed map
to perform suboptimal navigation.

Hence, for inter-layer edges, we devise a method to associate to each edge a weight that is at least as large as the shortest path of traversal. 
In particular,
for each inter-layer edge, we have a node in the higher layer,
denoted by $\mathbf{x}$,
and a node in the lower layer,
denoted by $\mathbf{y}$.
To find the weight of the inter-layer edge $ e_{\mathbf{x},\mathbf{y}} $ connecting $ \mathbf{x} $ and $ \mathbf{y} $,
we first find a node $\mathbf{y}_0$ in the lower layer that has the smallest Euclidean norm to the centroid of the set of all the nodes that are children of $\mathbf{x}$:
then,
the weight of $ e_{\mathbf{x},\mathbf{y}} $ is computed as 
\begin{equation}
	\weight{\dsg}{\mathbf{x}}{\mathbf{y}} \doteq \lVert\mathbf{x} - \mathbf{y}_0\rVert + \dist{\dsg}{\mathbf{y}_0}{\mathbf{y}}, %\text{path}_{\text{bfs}}(\mathbf{y}_0, \mathbf{y})
\end{equation}
where $ \lVert\cdot\rVert $ denotes the Euclidean norm.
% and $\text{path}_{\text{bfs}}(\cdot)$ denotes the path length of the path with the least number of edges between two nodes. 
Observe that the weight is greater than or equal to the path length of the shortest path between $\mathbf{y}_0$ and $\mathbf{y}$. 
Intuitively, the weight of a room-to-place inter-layer edge is the distance between the room centroid and the closest place, 
plus the path length from the closest place to the target place.

Importantly,
the above heuristic is consistent with navigation performance:
the cost (\ie estimated navigation time) of traversing the inter-layer edge between two nodes
(\eg to navigate from a place node to the room centroid) is higher than
following the shortest path on the place layer that connects the two corresponding place nodes
(\eg the source place node and the closest place node to the room centroid).
This feature is crucial in order for the proposed compression algorithms to work properly,
because they will first favor intra-layer edges (which retain navigation performance)
and use inter-layer edges only when necessary (causing an increase in navigation time).

However,
we also note that calculating the actual impact of local navigation is difficult having only the final \dsgshortname,
so that the actual navigation time may differ from the estimated one.
Nonetheless,
we argue that such inter-layer edge weights could be reliably estimated by a robot navigating the environment.
Indeed,
rather than calculating the weights \textit{a posteriori} based on the final \dsgshortname,
the robot can compute them \textit{online} (while building the \dsgshortname) based on its own exploration time:
this directly relates inter-layer edges between places,
rooms,
and buildings nodes
with the expected navigation effort.
		%!TEX root = ../main.tex

\section{Proof of~\cref{prop:bound-budlite}}\label{app:bound}

\begin{figure}
	\centering
	\includegraphics[width=.4\linewidth]{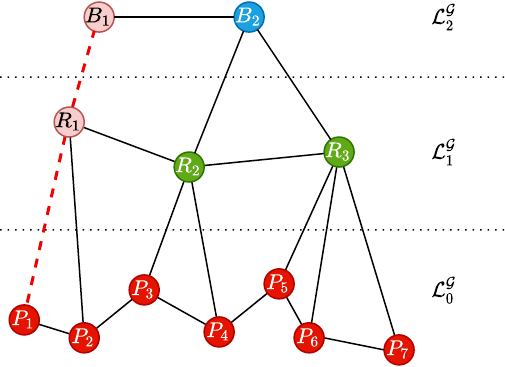}
	\caption{Partition of a toy \dsgname into layers.
		The ancestors of place node $P_1$ are colored in light red and connected through dashed red edges.
		For example,
		room node $R_1$ is the $1$st ancestor of $P_1$,
		that is,
		$\ancestor{P_1}{\dsg}{1} = R_1$.
		The three rooms have diameter equal to $2$, $2$, and $3$, respectively.
		For example,
		$\diam{R_3}{\dsg} = 3$ because $P_5$ and $P_7$,
		which are the nodes farthest apart among all its children,
		are connected by the shortest path $\{P_5,P_6,P_7\}$ with cardinality $3$.
	}
	\label{fig:layers}
\end{figure}

\begin{proof}
	First,
	recall that terminals are place nodes,
	that is,
	for any $(\source,\target)\in\pairs$,
	it holds $\source,\target\in\layer{\dsg}{0}$.
	
	\cref{alg:bottom-up} first builds a spanner with maximum distortion parameterized by $\beta$.
	Hence,
	the finest layer (with the smallest index) retained in such initial spanner needs to be either $\layer{\dsg}{\lz} $,
	with $\lz$ defined above,
	or any layer below.
	In particular,
	retaining $\layer{\dsg}{\lz} $ as the finest layer yields the maximum tolerable distortion given $\beta$.
	This can be seen from the following inequalities valid for any path $\path{\dsgc}{\source}{\target}$
	whose nodes are in layer $\layer{\dsg}{\ell}$ or in layers above:
	\begin{equation}\label{eq:first-layer}
		\dist{\dsgc}{\source}{\target} \ge 2\sum_{i=1}^{\ell}\wcmin{i} + \mmin{\ell}\wmin{\ell}
	\end{equation}
	where the summation accounts for cross-layer edges connecting nodes $\source$ and $\target$ to their respective ancestors in layer $\layer{\dsg}{\ell}$
	and the term $\mmin{\ell}\wmin{\ell}$ accounts for the minimum distance between such two ancestors.
	Then,
	after creation of the initial spanner,
	there exists at least one pair $(\source,\target)\in\pairs$ whose shortest path $\path{\dsgc}{\source}{\target}$
	passes across layer $\layer{\dsg}{\lz}$ or a layer below.
	Because the spanner construction is near-optimal and only guarantees an upper bound on distance stretch,
	it is possible that finer layers (with index smaller than $\lz$) are retained,
	providing lower distortion for shortest paths passing across them.
	
%	when the shortest path on the compressed graph passes through the $\ell$th layer $\layer{\dsg}{\ell}$,
%	this means that such a path starts from source node $\source$,
%	reaches the $\ell$th layer via cross-layer edges that connect $\source$ to its $\ell$th ancestor,
%	goes along nodes on the $\ell$th layer,
%	reaches the $\ell$th ancestor of node $\target$ on the $\ell$th layer,
%	and finally reconnects to node $\target$ via cross-layer edges between layers $\layer{\dsg}{\ell}$ and $\layer{\dsg}{0}$.

	\begin{figure}
		\centering
		\includegraphics[width=.45\linewidth]{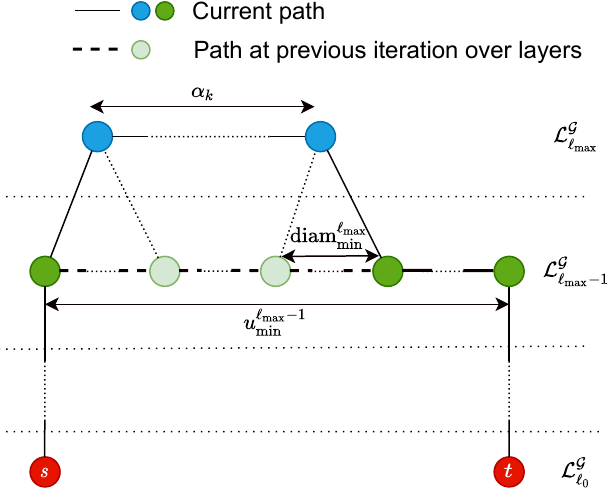}
		\caption{Representation of the maximum distortion provided by a shortest path in the compressed graph.
			Terminal pair $(\source,\target)\in\pairs$ is in the finest layer $\layer{\dsg}{\lz}$ retained by the spanner,
			which here corresponds to the place layer.
			After $k$ total iterations,
			the path has at most $\alpha_k$ nodes in the highest layer $\layer{\dsg}{\lmax}$.
			The remaining nodes in layer $\layer{\dsg}{\lmax-1}$ are at most the ones that remain
			after replacement of at least $\bmin{\lmax}$ nodes for each newly added node in layer $\layer{\dsg}{\lmax}$.
		}
	\end{figure}
	
	Consider now any number $k$ of total iterations of the innermost loop of~\cref{alg:bottom-up} at~\cref{alg1:foreach-path}
	that parses terminal pairs.
	After $\ell$ layers have been parsed in the outmost loop at~\cref{alg1:foreach-layer},
	with the last layer $\layer{\dsg}{\ell}$ having possibly been parsed only partially,
	the total number of iterations is lower bounded as follows:
	\begin{equation}\label{eq:lower-bound-iterations}
		k > |\pairs|\sum_{i=\lz}^{\ell-1}\mmin{i}
	\end{equation}
	where each pair has been parsed for at least $\sum_{i=\lz}^{\ell-1}\mmin{i}$ iterations to reach the $\ell$th layer.
	In particular,
	each pair is parsed for at least $\mmin{i}$ iterations in the innermost loop for each layer $\layer{\dsg}{i}$,
	$i=\lz,\dots,\ell-1$,
	plus other possible iterations for layer $\layer{\dsg}{\ell}$.
	Then,
	the highest layer $\layer{\dsg}{\ell}$ that can be reached after $k$ total iterations has index $\ell=\lmax$ defined above.
	Because higher layers feature coarser spatial resolution,
	the maximum distortion is given by the path with the largest distortion that passes across the highest reachable layer $\lmax$.
	This finally yields the upper bound in~\eqref{eq:bu-stretch-bound},
	as described next.
	The summation is the maximum distance given by cross-layer edges that connect terminals $\source$ and $\target$ to
	layer $\layer{\dsg}{\lmax}$.
	The maximum number of steps (nodes) of any path in such a layer is given by $\alpha_k$ defined above,
	whereby each layer $\layer{\dsg}{i}$ below $\layer{\dsg}{\lmax}$ provides the least number of steps
	between the corresponding $i$th ancestors of $\source$ and $\target$.
	Thus,
	the term $\alpha_k\wmax{\lmax}$ is the maximum length of the portion of the path in layer $\layer{\dsg}{\lmax}$.
	Finally,
	the last term is the maximum length of the remaining portion of the path,
	which passes across nodes in layer $\layer{\dsg}{\lmax-1}$ by construction.
	In particular,
	the maximum number of steps is given by $\mmin{\lmax-1}-\alpha_k\bmin{\lmax}$,
	that accounts for the remaining number of edges in layer $\layer{\dsg}{\lmax-1}$ after $\alpha_k$ iterations on the $\lmax$th layer.
\end{proof}
		%!TEX root = ../main.tex

\section{Extra Numerical Tests and Simulations Results}
\label{app:experiments}

In this Appendix,
we visually illustrate how the proposed algorithms perform on two tested simulated environments:
an apartment scene and the office scene already used in~\autoref{sec:experiments}.

\subsection{Apartment Scene}

\begin{figure}[h]
	\centering
	\begin{minipage}{.4\textwidth}
		\centering
		\includegraphics[width=.65\columnwidth]{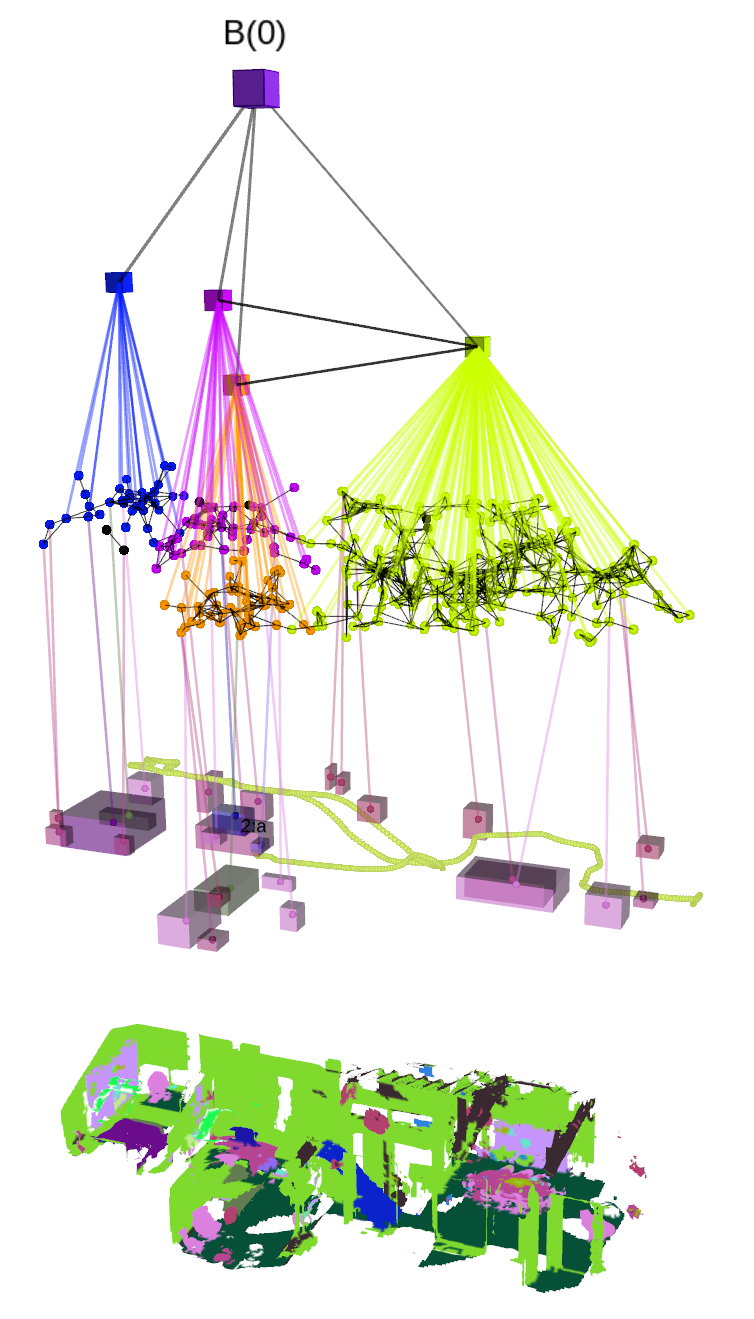}
	\end{minipage}%
	\hfil
	\begin{minipage}{.6\textwidth}
		\centering
		\includegraphics[width=.8\columnwidth]{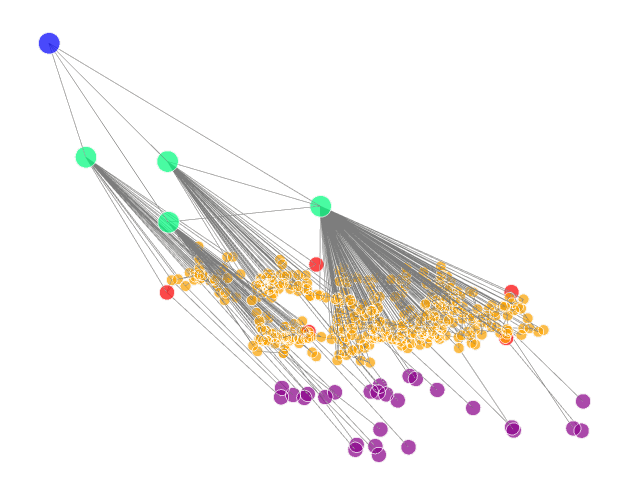}
	\end{minipage}%
	\caption{Original \dsgname of the Apartment scene.
		Left: full \dsgshortname with semantic map and meshgrid.
		Right: schematic version with objects (purple),
		places (yellow),
		rooms (green),
		and building (blue) nodes.
		Terminal places nodes are marked with red color.
		Full size: $ 453 $ nodes.
		\label{fig:apt-original}}
\end{figure}

Figure~\ref{fig:apt-original} shows the original \dsgshortname of the Apartment,
which is composed of $ 453 $ nodes connected by $ 1403 $ edges.

\begin{figure}
	\centering
	\begin{subfigure}{.333\textwidth}
		\centering
		\includegraphics[width=\columnwidth]{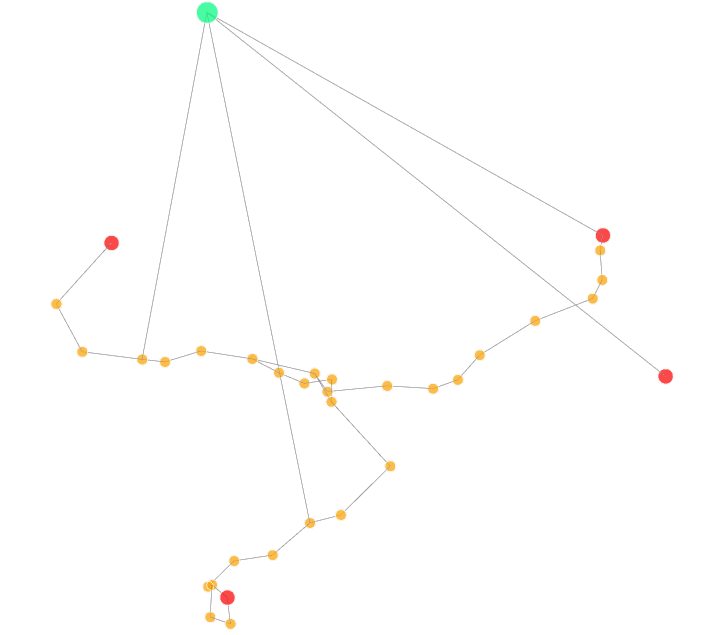}
		\caption{Budget: $ 40 $ nodes (actual size: $ 34 $).}
		\label{fig:apt-budlite-budget40}
	\end{subfigure}%
	\hfill
	\begin{subfigure}{.333\textwidth}
		\centering
		\includegraphics[width=\columnwidth]{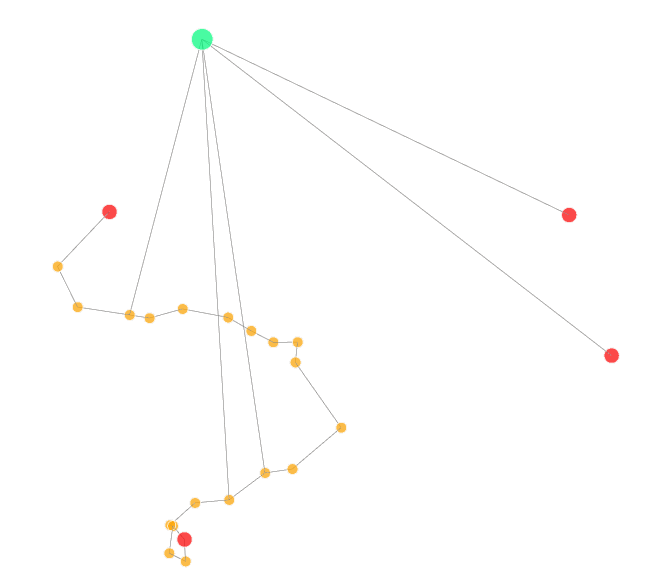}
		\caption{Budget: $ 30 $ nodes (actual size: $ 24 $).}
		\label{fig:apt-budlite-budget30}
	\end{subfigure}%
	\hfill
	\begin{subfigure}{.333\textwidth}
		\centering
		\includegraphics[width=\columnwidth]{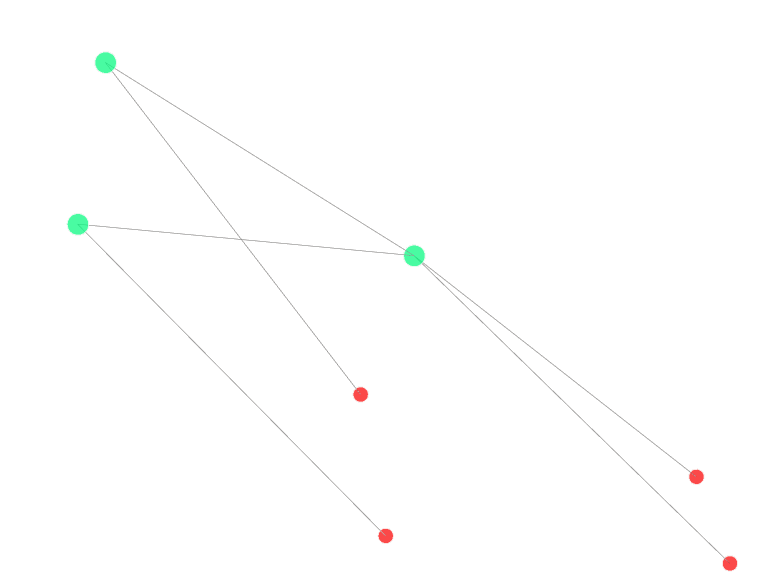}
		\caption{Budget: $ 10 $ nodes (actual size: $ 7 $).}
		\label{fig:apt-budlite-budget10}
	\end{subfigure}
	\caption{Compressed \dsgname of the Apartment output by \namebu.}
	\label{fig:apt-budlite}
\end{figure}

Figure~\ref{fig:apt-budlite} shows the compressed graphs obtained by running \namebu with different budget values.
Recall that \namebu compressed the \dsgshortname by exploiting the hierarchy bottom-up,
parsing shortest paths one after the other and abstracting away places nodes
to their corresponding room nodes.
In this case,
we consider four terminals nodes scattered across three rooms.
Note that,
as the available budget decreases
(\ie the communication constraint get tighter),
portions of shortest paths along places nodes are pruned away
and abstracted into their respective room nodes.
In particular,
\namebu uses a single room node when the budget is $ 30 $ (\autoref{fig:apt-budlite-budget40}) or larger (\autoref{fig:apt-budlite-budget30}),
sacrificing navigation performance for the two terminals nodes belonging to that room and
retaining fine-scale spatial information in proximity of the two terminals nodes belonging to other two rooms.
Conversely,
all rooms are used when the budget gets too small (\autoref{fig:apt-budlite-budget10}),
and the compression procedure is forced to remove all free-space locations in the place layer.
Importantly,
the output of \namebu also depends on the order in which terminals pairs (and hence shortest paths) are parsed,
which may cause larger or smaller amounts of nodes to be deleted before others:
improving this feature of the algorithm is an important aspect that will be considered in follow-up work.

\begin{figure}
	\centering
	\begin{subfigure}{.333\textwidth}
		\centering
		\includegraphics[width=\columnwidth]{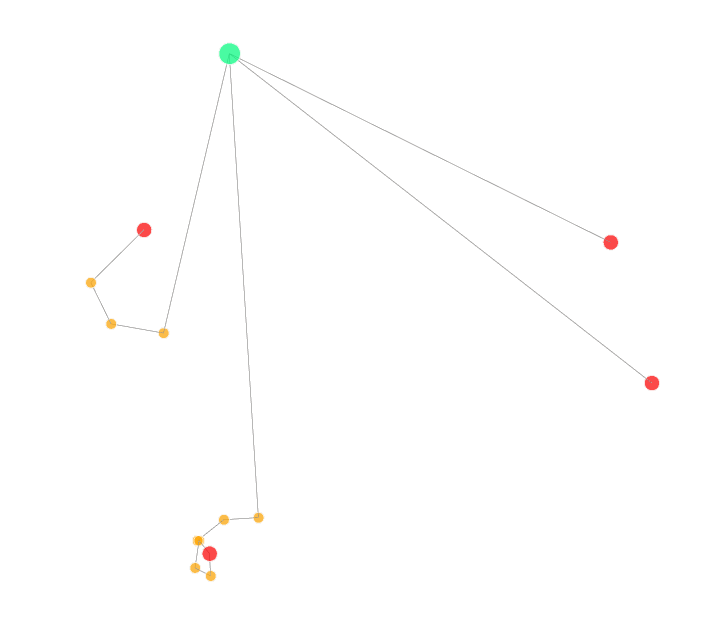}
		\caption{Budget: $ 40 $ nodes (actual size: $ 13 $).}
		\label{fig:apt-todlite-budget40}
	\end{subfigure}%
	\hfill
	\begin{subfigure}{.333\textwidth}
		\centering
		\includegraphics[width=\columnwidth]{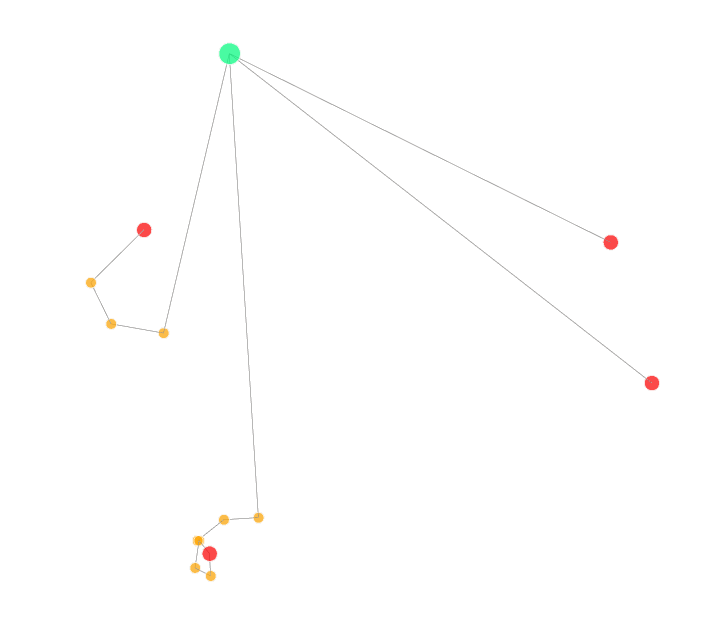}
		\caption{Budget: $ 30 $ nodes (actual size: $ 13 $).}
		\label{fig:apt-todlite-budget30}
	\end{subfigure}%
	\hfill
	\begin{subfigure}{.333\textwidth}
		\centering
		\includegraphics[width=\columnwidth]{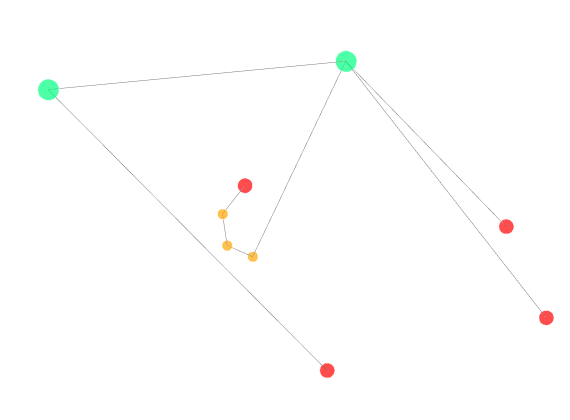}
		\caption{Budget: $ 10 $ nodes (actual size: $ 9 $).}
		\label{fig:apt-todlite-budget10}
	\end{subfigure}
	\caption{Compressed \dsgname of the Apartment output by \nametd.}
	\label{fig:apt-todlite}
\end{figure}

Conversely,
\nametd results are shown in~\autoref{fig:apt-todlite}.
Recall that this algorithm leverages the \dsgshortname hierarchy via top-down expansion of nodes,
which is clearly visible from~\autoref{fig:apt-todlite} as opposed to shortest path-wise compression on \namebu.
In particular,
it can be seen from~\cref{fig:apt-todlite-budget40,fig:apt-todlite-budget30} that
expanding one room node is not possible without exceeding the allowed budget:
hence,
\nametd expands the other two rooms nodes in both cases,
resulting in preserved places nodes close to the two terminals nodes on the left.
When budget is further reduced (\autoref{fig:apt-todlite-budget10}),
two rooms are not expanded
and fine-scale geometric information in the place layer is retained only for the room corresponding to the terminal node on the top left.

Comparing the compression results of \namebu in~\autoref{fig:apt-budlite} and of \nametd in~\autoref{fig:apt-todlite}
shows both their different mechanisms and advantages:
in this case,
large budgets favor the \namebu bottom-up compression,
which is able to retain more places nodes;
instead,
small budget favors the \nametd top-down expansion,
which eventually retains places nodes associated with one room,
whereas the order chosen to parse shortest paths in the \namebu forces to abstract away all nodes in the place layer.

\subsection{Office Scene}

\begin{figure}[h]
	\centering
	\begin{minipage}{.333\textwidth}
		\centering
		\includegraphics[width=.9\columnwidth]{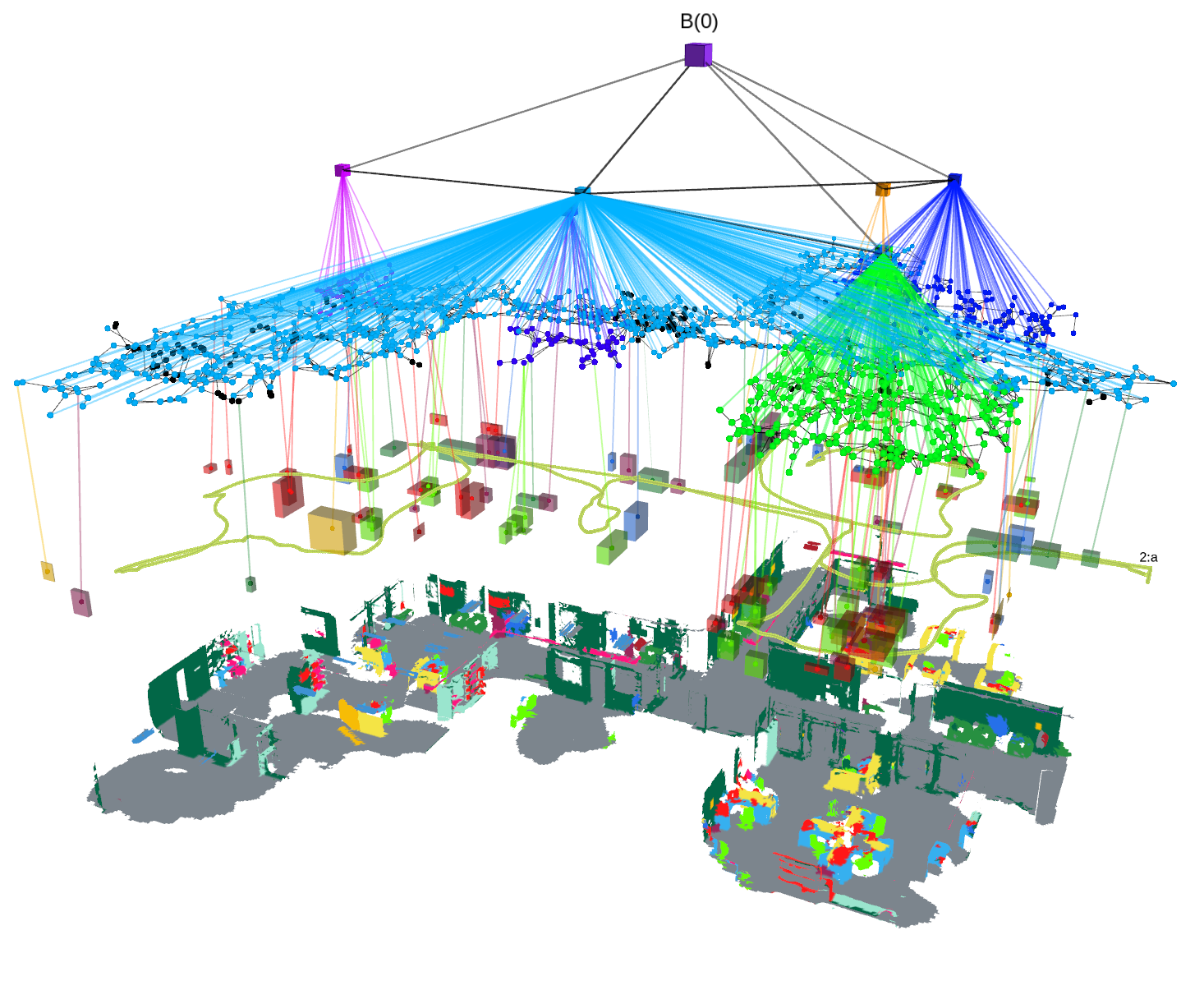}
	\end{minipage}%
	\hfill
	\begin{minipage}{.333\textwidth}
		\centering
		\includegraphics[width=\columnwidth]{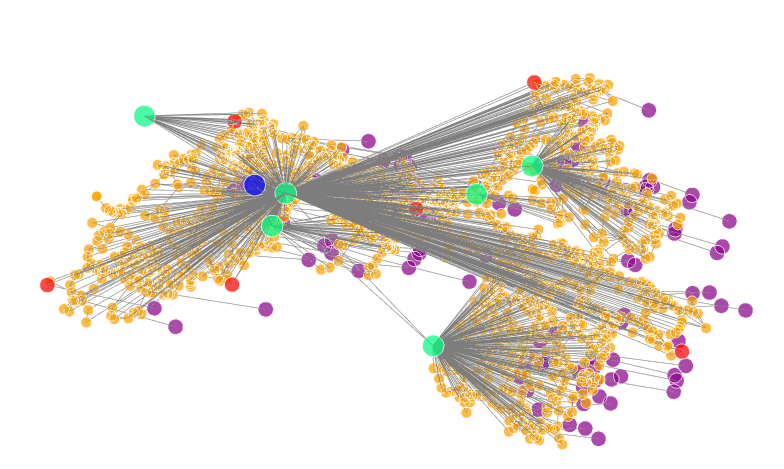}
	\end{minipage}%
	\hfill
	\begin{minipage}{.333\textwidth}
		\centering
		\includegraphics[width=\columnwidth]{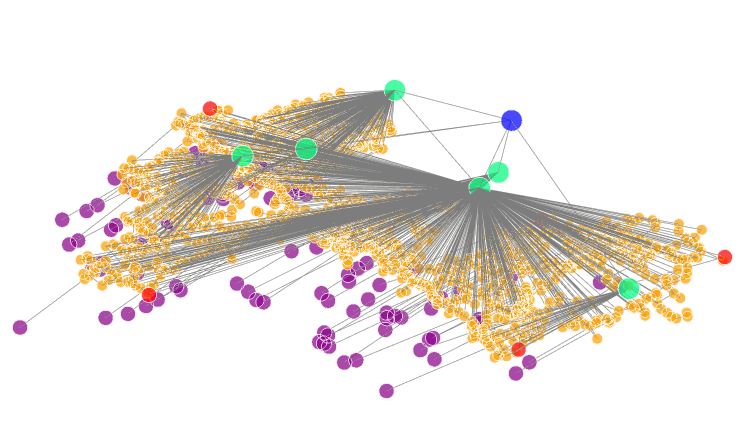}
	\end{minipage}
	\caption{Original \dsgname of the Office scene.
		Left: full \dsgshortname with semantic map and meshgrid.
		Right: schematic version.
		Full size: $ 1675 $ nodes.
		\label{fig:office-original}}
\end{figure}

Figure~\ref{fig:office-original} shows the original \dsgshortname of the Office,
which is composed of $ 1675 $ nodes connected by $ 5396 $ edges.
For this test,
we consider six terminal nodes scattered across two rooms.

\begin{figure}
	\centering
	\begin{subfigure}{.333\textwidth}
		\centering
		\includegraphics[width=\columnwidth]{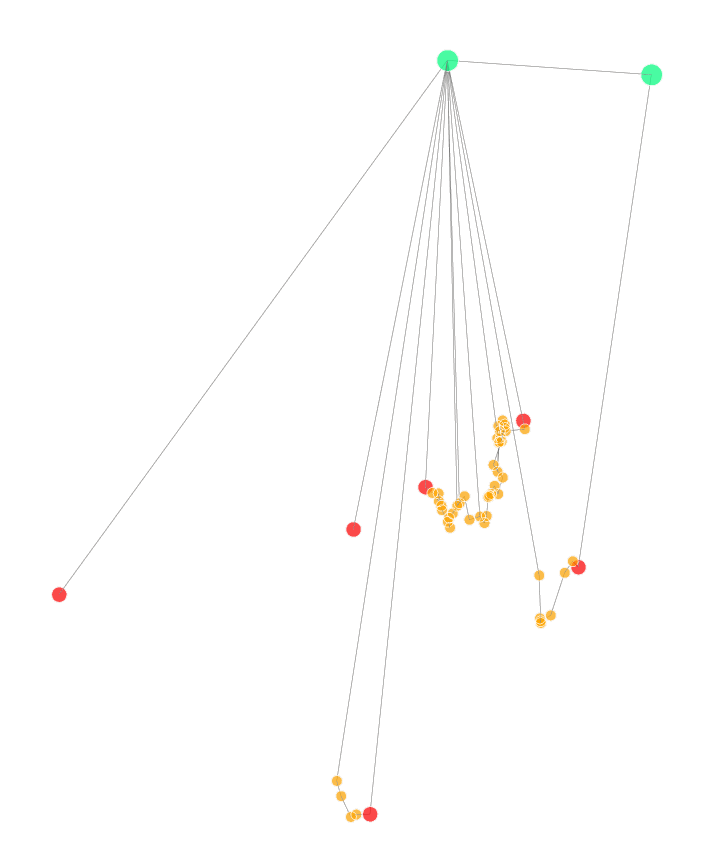}
		\caption{Budget: $ 60 $ nodes (actual size: $ 55 $).}
		\label{fig:office-budlite-budget60}
	\end{subfigure}%
	\hfill
	\begin{subfigure}{.333\textwidth}
		\centering
		\includegraphics[width=\columnwidth]{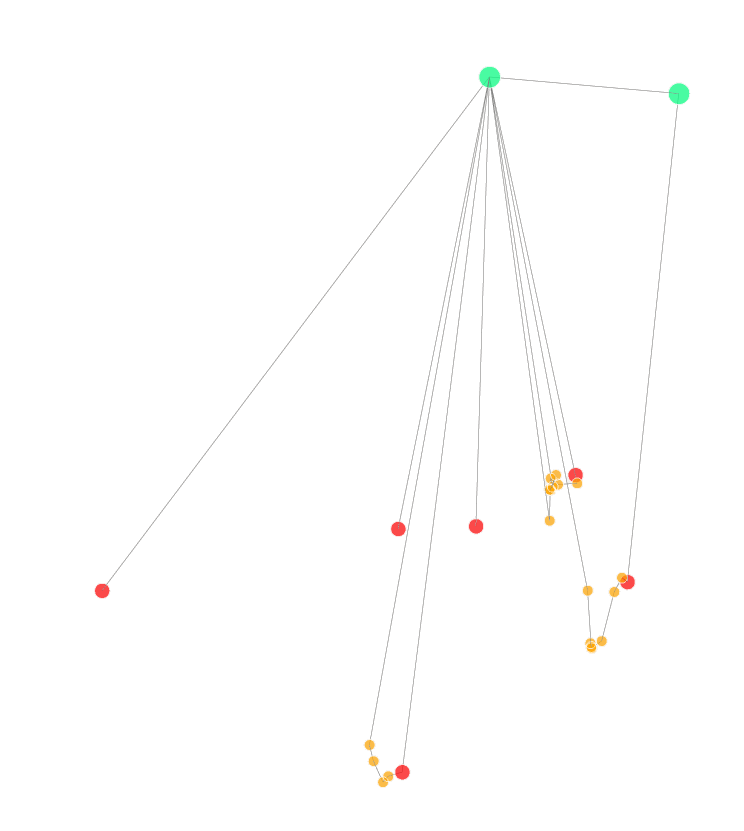}
		\caption{Budget: $ 30 $ nodes (actual size: $ 24 $).}
		\label{fig:office-budlite-budget30}
	\end{subfigure}%
	\hfill
	\begin{subfigure}{.333\textwidth}
		\centering
		\includegraphics[width=\columnwidth]{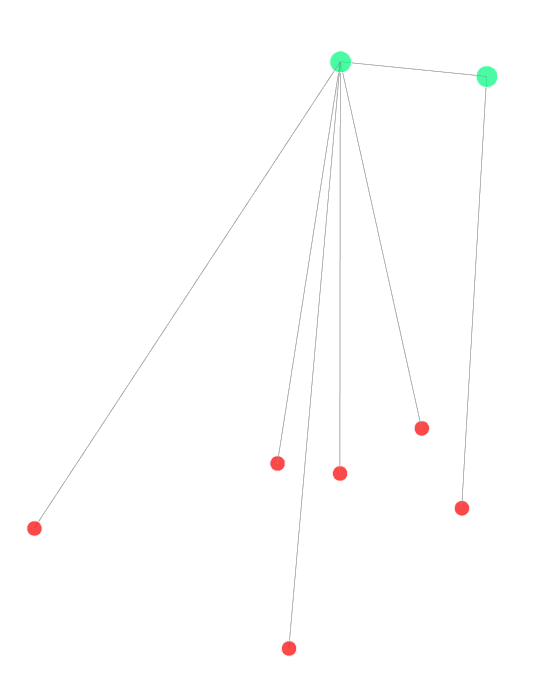}
		\caption{Budget: $ 10 $ nodes (actual size: $ 8 $).}
		\label{fig:office-budlite-budget10}
	\end{subfigure}
	\caption{Compressed \dsgname of the Office output by \namebu.}
	\label{fig:office-budlite}
\end{figure}

Figure~\ref{fig:office-budlite} shows the compressed graphs output by \namebu.
The same general remarks carried out before also apply here.
Notably,
one of the interested rooms (gathering five out of the six terminals nodes) is very large
(it is in fact a corridor, see~\autoref{fig:office-original}),
which may cause the compression procedure of \namebu to act in too unbalanced fashion
if the portions of shortest paths passing through that room are abstracted away at once.
To improve granularity of compression in this case,
we forced a maximum number of places nodes that can be compressed within a single iteration
(corresponding to a slight modification to the condition of~\cref{alg1:check-children} of~\cref{alg:bottom-up}):
in particular,
we set $ 20 $ places nodes as maximum threshold,
so that long stretches of places nodes are compressed at a pace of $ 20 $ (or fewer) at each iteration.

\begin{figure}
	\centering
	\begin{subfigure}{.25\textwidth}
		\centering
		\includegraphics[width=.8\columnwidth]{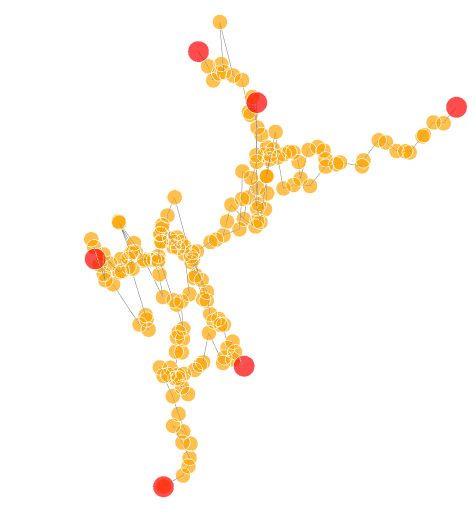}
		\caption{Initial condition.}
		\label{fig:office-steps-budlite-0}
	\end{subfigure}%
	\hfill
	\begin{subfigure}{.25\textwidth}
		\centering
		\includegraphics[width=.8\columnwidth]{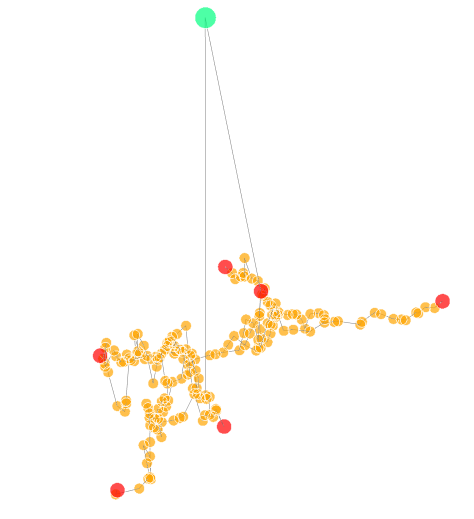}
		\caption{Iteration $ 1 $.}
		\label{fig:office-steps-budlite-1}
	\end{subfigure}%
	\hfill
	\begin{subfigure}{.25\textwidth}
		\centering
		\includegraphics[width=.8\columnwidth]{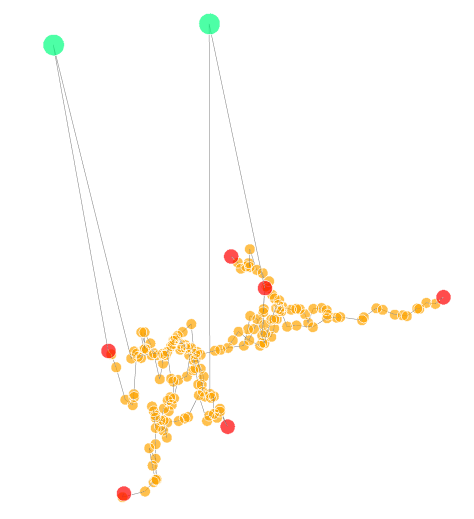}
		\caption{Iteration $ 2 $.}
		\label{fig:office-steps-budlite-2}
	\end{subfigure}%
	\hfill
	\begin{subfigure}{.25\textwidth}
		\centering
		\includegraphics[width=.8\columnwidth]{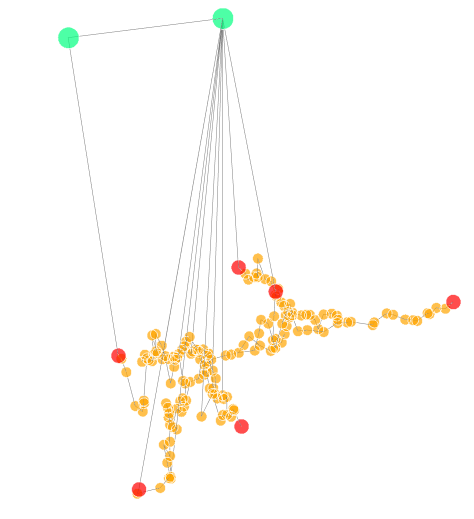}
		\caption{Iteration $ 10 $.}
		\label{fig:office-steps-budlite-10}
	\end{subfigure}\\
	\begin{subfigure}{.25\textwidth}
		\centering
		\includegraphics[width=.8\columnwidth]{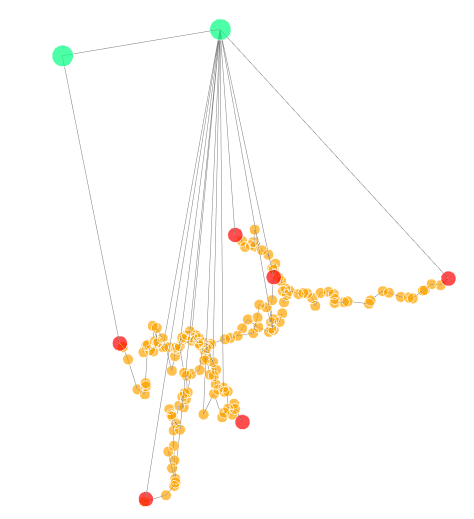}
		\caption{Iteration $ 11 $.}
		\label{fig:office-steps-budlite-11}
	\end{subfigure}%
	\hfill
	\begin{subfigure}{.25\textwidth}
		\centering
		\includegraphics[width=.8\columnwidth]{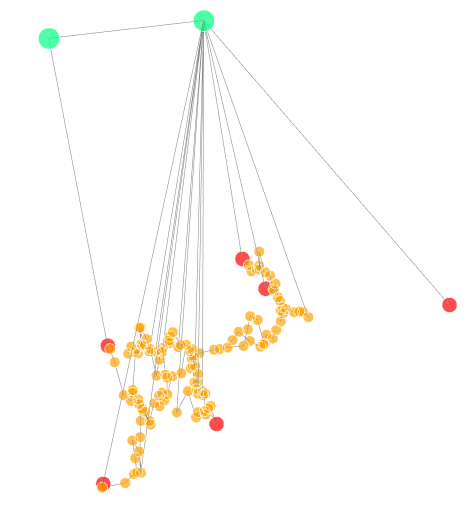}
		\caption{Iteration $ 15 $.}
		\label{fig:office-steps-budlite-15}
	\end{subfigure}%
	\hfill
	\begin{subfigure}{.25\textwidth}
		\centering
		\includegraphics[width=.8\columnwidth]{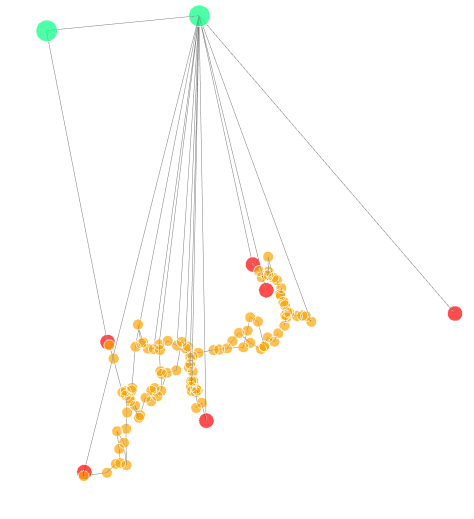}
		\caption{Iteration $ 20 $.}
		\label{fig:office-steps-budlite-20}
	\end{subfigure}%
	\hfill
	\begin{subfigure}{.25\textwidth}
		\centering
		\includegraphics[width=.8\columnwidth]{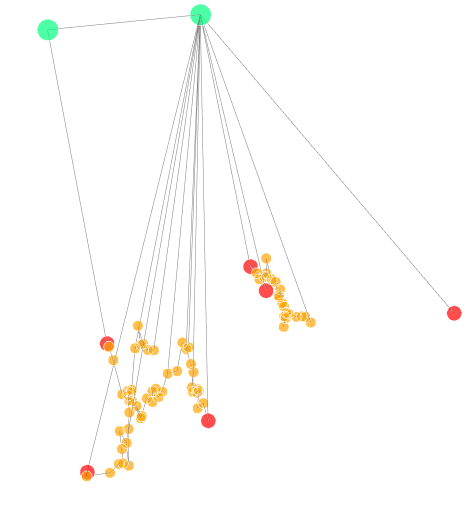}
		\caption{Iteration $ 25 $.}
		\label{fig:office-steps-budlite-25}
	\end{subfigure}
	\caption{Iterations of the bottom-up compression phase of \namebu (budget: $ 60 $ nodes).}
	\label{fig:office-steps-budlite}
\end{figure}

Figure~\ref{fig:office-steps-budlite} shows the breakdown of some iterations of the compression procedure carried out within
the \namebu algorithm,
corresponding to the loop at~\cref{alg1:foreach-path} of~\cref{alg:bottom-up}.
The initial condition shown in~\autoref{fig:office-steps-budlite-0} corresponds to the spanner output by \texttt{build\_spanner}
in~\cref{alg1:initialization}.
Note that the latter is composed of only places nodes,
and rooms nodes abstractions are introduced by subsequent compression iterations.
Specifically,
one room is added at the first iteration (\autoref{fig:office-steps-budlite-1})
and the other,
which is connected to the first room,
at the second iteration (\autoref{fig:office-steps-budlite-2}).
Shortest paths are parsed one after the other,
which causes places nodes to be retained until there is no path using them:
for example,
the inter-layer edge between the rightmost terminal node and its associated room node is added at iteration $ 11 $ (\autoref{fig:office-steps-budlite-11}),
but the corresponding stretch of places nodes is removed only at iteration $ 15 $ (\autoref{fig:office-steps-budlite-15}),
when all shortest paths with source the rightmost terminal node have been parsed and shortcut through the room.
The output compressed \dsgshortname in~\autoref{fig:office-budlite-budget60} is obtained after $ 28 $ iterations.

\begin{figure}
	\centering
	\begin{subfigure}{.333\textwidth}
		\centering
		\includegraphics[width=\columnwidth]{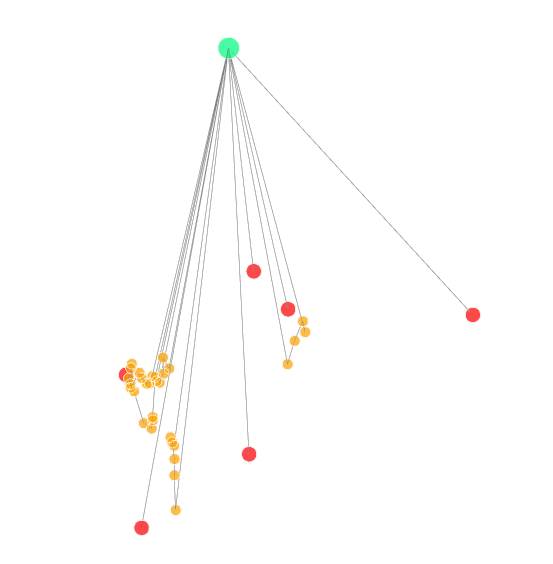}
		\caption{Budget: $ 60 $ nodes (actual size: $ 38 $).}
		\label{fig:office-todlite-budget40}
	\end{subfigure}%
	\hfill
	\begin{subfigure}{.333\textwidth}
		\centering
		\includegraphics[width=.9\columnwidth]{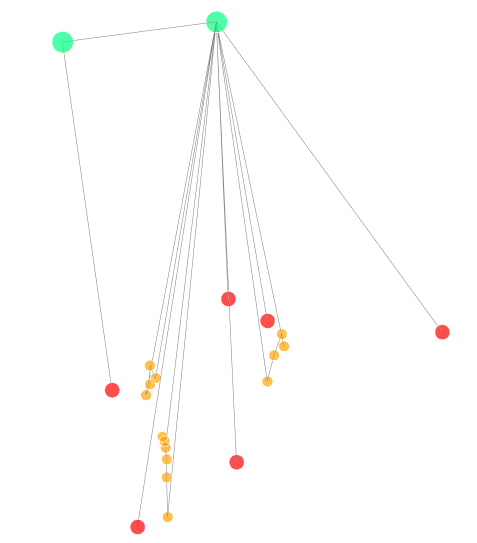}
		\caption{Budget: $ 30 $ nodes (actual size: $ 22 $).}
		\label{fig:office-todlite-budget30}
	\end{subfigure}%
	\hfill
	\begin{subfigure}{.333\textwidth}
		\centering
		\includegraphics[width=\columnwidth]{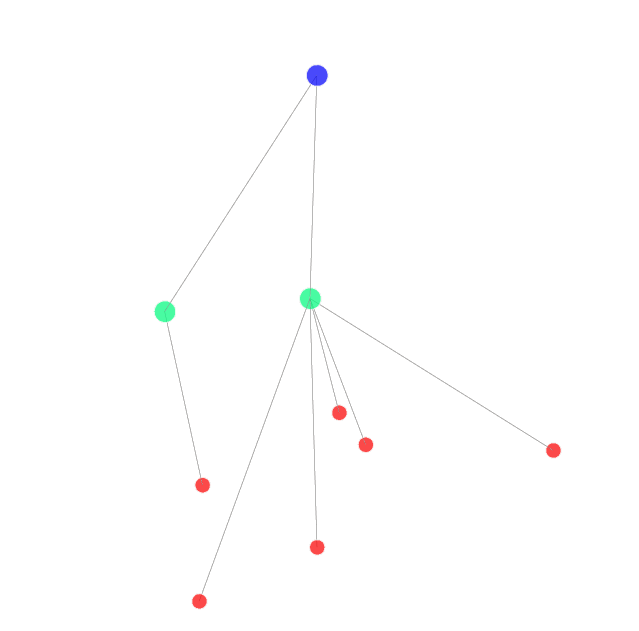}
		\caption{Budget: $ 10 $ nodes (actual size: $ 9 $).}
		\label{fig:office-todlite-budget10}
	\end{subfigure}
	\caption{Compressed \dsgname of the Office output by \nametd.}
	\label{fig:office-todlite}
\end{figure}

Figure~\ref{fig:office-todlite} shows the compressed graphs output by \nametd.
The same general remarks carried out for the Apartment also apply here.
In particular,
note that the small budget of $ 10 $ nodes in this case
prevents \nametd to perform any expansion,
see~\autoref{fig:office-todlite-budget10}.
	}
	
\end{document}